\documentclass{mcom-l}

\usepackage{hyperref}
\usepackage{bigints}
\newcommand{\R}{\mathbb{R}}
\theoremstyle{definition}
\usepackage{amsmath}
\usepackage{multirow}
\usepackage{algorithm}
\usepackage{algorithmic}
\usepackage{verbatim}
\usepackage{slashbox}

\usepackage{xparse}
\usepackage[autostyle]{csquotes}
\usepackage{amssymb}

\usepackage{xcolor}
\usepackage{mathtools}
\newtheorem{theorem}{Theorem}
\newtheorem{lemma}[]{Lemma}
\newtheorem{definition}[]{Definition}

\newtheorem{corollary}{Corollary}[theorem]

\theoremstyle{remark}

\numberwithin{equation}{section}

\begin{document}

\title[Hierarchical learning for noisy datasets]{Hierarchical regularization networks for sparsification based learning on noisy datasets}


\author{Prashant Shekhar}
\address{Data Intensive Studies Center, Tufts University, Medford, MA, 02155}
\curraddr{}
\email{prashant.shekhar@tufts.edu}
\thanks{}

\author{Abani Patra}
\address{Data Intensive Studies Center, Department of Mathematics, Department of Computer Science, Tufts University, Medford, MA, 02155}
\curraddr{}
\email{abani.patra@tufts.edu}
\thanks{This work was funded by the grants NSF1821311, NSF1645053, NSF1621853.}


\date{June 9, 2020}

\dedicatory{}

\begin{abstract}
We propose a hierarchical learning strategy aimed at generating sparse representations and associated models for large noisy datasets. The hierarchy follows  from  approximation spaces identified at successively finer scales. For promoting model generalization at each scale, we also introduce a novel, projection based penalty operator across multiple dimension, using permutation operators for incorporating proximity and ordering information. The paper  presents a detailed analysis of approximation properties in the reconstruction Reproducing Kernel Hilbert Spaces (RKHS) with emphasis on optimality and consistency of predictions  and behavior of error functionals associated with the produced sparse representations. Results show the performance of the approach as a data reduction and modeling strategy on both synthetic (univariate and multivariate) and real datasets (time series). The sparse model for the test datasets, generated by the presented approach, is also shown to efficiently reconstruct the underlying process and preserve generalizability.
\end{abstract}

\maketitle

\section{Introduction}
Hierarchical learning traditionally involves a sequence of operations based on some hierarchy, for making useful inferences from data. Bayesian hierarchical models for example usually involve a hierarchy of three model classes, the data model, the process model and finally the parameter model \cite{berliner1996hierarchical,cressie2015statistics}. This forms a hierarchy for the updating scheme of the parameters as learning happens sequentially over time. Multiscale models also have an inbuilt hierarchy of approximations, and various research works try to make joint inference on data, by combining these model components in some intelligent fashion \cite{bermanis2013multiscale,floater1996multistep}. Hierarchical models also have parallels to deep learning models which implement sequential function compositions to learn a data generation mechanism \cite{1901.02220,goodfellow2016deep}.

Motivated by these  diverse applications, we present a hierarchical structure of competing regularization networks \cite{evgeniou2000regularization,poggio1990regularization,poggio1990networks}, that make inferences over the observed data. The chosen network has to satisfy the criteria of highest generalizable performance with least model complexity \cite{forster2000key}. The requirement of least complexity also allows for generation of a sparse representation for the dataset, making our approach suitable for data reduction problems \cite{czarnowski2018approach,ur2016big,yildirim2016parallel}. 
\footnote{ The code for the proposed approach is available online \url{https://github.com/pshekhar-tufts/Hierarchical_noisy.git}}Our approach introduces a scale parameter $s $ and defines a mapping between $s$ and the corresponding approximation space $\mathcal{H}_s$ in the hierarchy of spaces considered. The main idea of exploiting the inherent correlation structure in the data at multiple levels follows directly from \cite{1906.11426}. However, the notion of convergence used in \cite{1906.11426} fails if the observations are reported with sampling noise. We have addressed the problem of sparse modeling for such noisy datasets in a similar hierarchical setting.

\subsection{Problem setup and definition}
Let $Y = (y_1,y_2,..,y_n) \in \R^n$  are discrete data values observed at $X = (x_1, x_2,....,x_n) \in \R^{n \times d}$. Considering some true underlying process $f: \R^d \to \R$, the values in Y can be regarded as noisy versions of $f|_X$ ($y_i = f(x_i) + noise$). We further consider two additional sets. First set $\Omega_x$ contains the data points $x_i$ at which the observations were made ($x_i \in X \subset \Omega_x \subset \R^{d}$). The observations are samples from a second set, $ \Omega_y$ ($y_i \in Y \subset  \Omega_y \subset \R$). Now, for a fixed element of $\Omega_x$, we expect a probabilistic distribution on $ \Omega_y$. Hence a joint probability distribution $p(x,y)$ can be defined on $\Omega_x \times  \Omega_y$. Therefore our training data $D = \{(x_i,y_i) \in \Omega_x \times  \Omega_y\}_{i=1}^n$ can be thought of as a result of $n$ samples ($i.i.d$) from $\Omega_x \times  \Omega_y$ according to the distribution $p(x,y)$. 

Given such a random noisy data sample $D$, we propose a strategy for data reduction and learning through intelligent sparsification. Data reduction seeks to find a smaller sparse subset $X_s \subseteq X$, that is sufficient for providing acceptable approximations to the underlying process $f|_X$ while also generalizing predictions to unseen data points $x (\in \Omega_x) \not \in X$. The learning part is justified by the sparse model produced by the proposed approach, that exclusively uses the subset $X_s$ to make these predictions. Hence in essence, our approach makes the following transformation to the input data
\begin{equation}
    full\ dataset \Rightarrow sparse\ representation + sparse\ model
    \label{largesmall}
\end{equation}

Therefore, the proposed approach can be used to replace large noisy datasets with a smaller subset and an associated model that can be used to make all future predictions. The strategy may also be used to construct effective surrogates of complex computer models by sampling outputs. We note the strategy is provably good for prediction in the domain of observation.  

\subsection{Proposed solution framework}
Given such a problem setup,  we are required to learn a function $\hat{f} \in \mathcal{H}$ (native Reproducing Kernel Hilbert Spaces ( RKHS)) which is closest (within some measure) to being the underlying process generating observations Y at X. For dealing with the ill-posedness of the problem of fitting noisy data, additional smoothness constraints are applied. Thus we have the following variational problem as our objective

\begin{equation}\label{eq1}
\hat{f}= \operatorname*{arg\min_{\tilde{f} \in \mathcal{H}}} \Bigg[ \frac{1}{n} V(Y,\tilde{f}|_X)+ \lambda \cdot \zeta(\tilde{f}) \Bigg]
\end{equation}

Here $V(\cdot,\cdot)$ is a loss function and $\zeta(\tilde{f}) = ||Z\tilde{f}||^2_{\mathcal{H}}$ is called a stabilizer, where $Z$ usually is a differential operator and $||\cdot||_\mathcal{H}$ is the native RKHS norm. For example, if we make the following choices in 1-dimension

\begin{equation}\label{eq2}
V(Y,\tilde{f}|_X) = \sum_{i=1}^n(y_i -\tilde{f}_i)^2 \quad \text{and }\zeta(\tilde{f}) = ||Z \tilde{f}||^2 = \bigintsss_{\Omega_x} \Bigg[\frac{d^2 \tilde{f}(x)}{dx^2}\Bigg]^2dx
\end{equation}

then the function which minimizes (\ref{eq1}) is a spline \cite{wahba1990spline,green1993nonparametric}. Also $\lambda$ in (\ref{eq1}) is the regularization parameter which maintains a balance between approximation accuracy and smoothness. We obtain the classical ($L_2$) regularization network if we use squared error loss (as in (\ref{eq2})) in formulation (\ref{eq1}). \cite{poggio1990regularization} revealed this relationship between algorithms implementing regularization induced smoothness, with Multilayer Neural Networks. 

Work presented here proposes to extend the hierarchical algorithm from \cite{1906.11426} to noisy datasets by solving the variational problem (\ref{eq1}) at multiple scales (equivalent to fitting multiple competing regularization networks) and inferring the network (indexed by scale) that is most appropriately able to model the observations reported. The measure of `appropriateness' will be discussed in more detail in the following sections. Given the random sample of data $D = \{(x_i,y_i) \in \Omega_x \times  \Omega_y\}_{i=1}^n$, our approach considers a sequence of scale dependent RKHS $\mathcal{H}_s$, with an associated kernel $K^s: \R^d \times \R^d \rightarrow \R$, allowing (for each scale s) us to write a noisy data model of the form

\begin{equation}\label{model1}
Y = \mathcal{T}^sf + \varepsilon   
\end{equation}

Here $\varepsilon \sim N(0,\sigma^2_{\varepsilon}I) \in \R^n$ is a generic error term at each scale,  with function $f$ $ \in \mathcal{H}_s$ (assumed) being the true latent process to be inferred. $\mathcal{T}^s$ is an evaluation functional defined as $\mathcal{T}^sf = (f_1, f_2,...f_n)^T \in \R^n$. As evaluation functionals are bounded and linear in RKHS, therefore $\mathcal{T}^s \in \mathcal{B}(\mathcal{H}_s,\R^n)$. Hence given data $D$, our approach fits the model of type (\ref{model1}) by considering a sequence of scale dependent approximation spaces $\mathcal{H}_s$ (to infer $f \in \mathcal{H}_s$).  More specifics on $\mathcal{H}_s$ are provided in the subsequent sections. The scale $s$ with the $best\ approximation$ ($A_sf:\R^d \to \R$ where $A_sf \in \mathcal{H}_s$) to $f$ (among the discretized scales considered in the scale space) is then returned as the convergence scale ($t$).

While generating scale dependent models for the data, our hierarchical approach also creates a series of corresponding sparse subsets ($X_1, X_2,....,X_s,..$) which consist of $representative$ data points from X ($X_s \subseteq X$) \cite{chaudhuri1994finding,tejada2016selection} chosen intelligently by the algorithm. The cardinality (number of data points) of these subsets follow the relation
\[
|X_1| \leq |X_2| \leq .... \leq |X_s| \leq ..|X|\text{;} \quad \text{where $|\cdot|$ is the cardinality operator}
\]

Here, it should be noted that the the approximations $(A_sf): \R^d \to \R$ at each scale only use the datapoints in the corresponding sparse subset $X_s$. This enables efficient inference from a reduced version of the original dataset D and justifies the transformation in (\ref{largesmall}).

The scope of application of the ideas presented in this paper is  general   in both  problems targeted and proposed approach, with relations to many other research problems.  For example, multiresolution analysis provides one of the earliest references on multiscale processing of datasets \cite{mallat1989theory,daubechies1992ten}. There is also a rich literature on geometric data analysis with diffusion maps incorporating the ideas of multiscale analysis \cite{coifman2005geometric,coifman2006diffusion,maggioni2005biorthogonal}.
The hierarchy in our approximation spaces is closely related to Hierarchical Radial Basis Functions (HRBF) \cite{ferrari2004multiscale,borghese1998hierarchical}. These research works focus on combining models at multiple scales to appropriately capture an underlying process. This idea of multiscale basis functions also forms the foundation in more recent works like \cite{bermanis2013multiscale}, where the authors project the error orthogonal to the approximation space of previous scales to the next scale. This idea was also explored before by \cite{floater1996multistep}. For our problem, since we are targeting noisy data,  instead of combining the scales to reduce fitting error, we consider one scale at a time and incorporate an additional regularization parameter that promotes generalization. Since our approach generates data driven hierarchical basis functions belonging to RKHS, therefore the proposed approach is also related to work such as \cite{chen2013multi} and \cite{allard2012multi}, where the authors consider data dependent multiscale dictionaries that generalize wavelets in geometric sense. The physics based models have utilized the idea of multilevel modeling through multigrid methods \cite{briggs2000multigrid,stuben2001review}. There are many related papers in the general field of data analysis and machine learning (see for e.g. \cite{galun2015review,kushnir2009efficient} ) relating the idea to our approach. Since, the current work focuses on generating hierarchical basis functions, it is also closely related to works such as \cite{bohn2016sparse,griebel2014sparse} that implement the idea of sparse grids for data analysis and learning tasks.

\subsection{Contributions}
The principal contributions of this paper can be summarized as follows:
\begin{itemize}
    \item A hierarchical approach to data reduction and modeling using a sparse representation of the dataset is introduced. This enables us to replace a large noisy datasets with its sparse representation and an associated model for making any future predictions and generalizations. 
    \item The paper also proposes a novel type of smoothing penalty in multiple dimensions based on projections. This is achieved through a set of permutation operators for implementing localized penalties of varying degree.
    \item The paper also develops and presents theoretical foundations for the approximation and consistency properties of the proposed algorithm. This is followed by a detailed analysis of bounds on approximation operators and error in mean approximations.
\end{itemize}

\section{Hierarchical learning approach}
In our previous work \cite{1906.11426}, building on the work in \cite{bermanis2013multiscale}, we introduced and developed a methodology of data reduction (for noiseless data) through efficient basis construction exploiting the correlation structure present in the data. This algorithm was based on getting a relevant set of trial functions sampled as columns from a discrete kernel function. The scale at which these basis functions were able to efficiently approximate the observed data in the least square sense was considered as the convergence scale. The approach  constructed a sequence of scale (s) dependent approximations (represented as $(A_1f),(A_2f), (A_3f),..,(A_sf),...$) to the unknown function $f: \R^{d} \to \R$ by considering a hierarchy of approximation spaces $\mathcal{H}_s$. Each of these approximations used a subset of dataset $X_1, X_2,X_3,...X_s,.. \subseteq X$ respectively for learning. Since the data was noiseless ($f|_X$ was directly observed instead of $Y$), the target function was projected on the sampled basis at each scale by solving the following optimization problem

\begin{equation}\label{eq3}
A_sf =  \operatorname*{arg\min_{\tilde{f} \in \Gamma^s}} \Big[V(f|_X,\tilde{f}|_X)\Big] = \operatorname*{arg\min_{\tilde{f} \in \Gamma^s}} \Big[ ||f|_X - \tilde{f}|_X||^2_2  \Big]
\end{equation}

Here $\Gamma^s$ is the subspace defined at each scale s in the native RKHS as 
\begin{equation}\label{eq4}
\Gamma^s = span\{K^s(.,x_i): x_i \in X_s\} \approx span\{K^s(.,x_j): x_j \in X\} \quad X_s \subseteq X
\end{equation}

with $K^s(\cdot,\cdot)$ being the reproducing kernel for the RKHS $\mathcal{H}_s$ \cite{buhmann2003radial,wendland2004scattered,aronszajn1950theory} and formulation (\ref{eq3}) being the standard problem of orthogonal projection \cite{saad2003iterative}.

In the this paper we extend this idea to noisy datasets, where models cannot rely completely on the observations (as they are corrupted with noise). So we ameliorate the effect of noise by introducing a penalty function for inducing smoothness (under the common assumption that noise induces false rapid fluctuations \cite{wahba1990spline}) thus obtaining the following constrained projection formulation (same as the $L_2$ regularization network functional as in (\ref{eq1})). 

\begin{equation}\label{eq5}
A_sf =  \operatorname*{arg\min_{\tilde{f} \in {\Gamma^s}}} \Bigg[ \frac{1}{n} ||Y - \mathcal{T}^s \tilde{f}||^2_2 + \lambda_s ||J_s \tilde{f}||^2_{\mathcal{H}_s}  \Bigg]\
\end{equation}

Here $J_s$ is a suitable projection operator on $[\Gamma^s]$ (a particular choice of $\zeta(\cdot)$) which allows efficient penalization (regularization) of sharp changes in $\tilde{f}$. $\mathcal{T}^s$ is a evaluation functional defined in (\ref{model1}). The solution to (\ref{eq5}) has a form $A_sf = \sum_{i = 1}^{|X_s|}\hat{\theta}_i K^s(\cdot,x_i)$ (from  \cite{pearce2006penalized}, $x_i \in X_s$), with $\hat{\theta}_i$ being suitable basis weights minimizing the cost objective \ref{eq5} and $K^s$ being the reproducing kernel for $\mathcal{H}_s$.

\subsection{Regularization structure}
Following standard procedures in kernel based approximation methods \cite{pearce2006penalized, wahba1990spline}, it is often desirable to only penalize certain specific functions in $\mathcal{H}_s$ and keep the rest of the functions unpenalized (which is achieved precisely by the projection operator $J_s$ in (\ref{eq5})). Let $\mathcal{H}_{s,0} = span\{\psi_0, \psi_1,....,\psi_{p_1}\}$ be a subspace of $\mathcal{H}_s$ containing these unpenalized functions with its orthogonal complement $\mathcal{H}_{s,1}$ $(\mathcal{H}_{s,1} = \mathcal{H}^{\perp}_{s,0}= span\{\phi_0, \phi_1,....,\phi_{p_2}\})$ spanned by the functions whose behavior needs to be constrained. Therefore $\mathcal{H}_s = \mathcal{H}_{s,0} \bigoplus  \mathcal{H}_{s,1}$ (also $p_1 + p_2 = |X_s|$). Coming back to (\ref{eq5}), we conclude that a suitable $J_s$ has $\mathcal{H}_{s,0}$ as its null space with  $\mathcal{H}_{s,1}$ being its projection or range space. \cite{aronszajn1950theory} also showed that $\mathcal{H}_{s,0}$ and $\mathcal{H}_{s,1}$ are themselves valid RKHS with suitable Kernels $K^s_0$ and $K^s_1$ respectively such that $K^s = K^s_0 + K^s_1$. The projection operator $J_s$ can take various forms \cite{green1993nonparametric,eilers1996flexible,tibshirani2015statistical},  however for our hierarchical approach we have chosen to implement a difference operator based penalty on the projections across each dimension (similar to the one used by \cite{eilers1996flexible}). For better understanding of the penalty operator, consider a Relation R ( $\leq$: less-than-or-equal) \cite{oden2017applied} defined on the domain set $\Omega_x \subset \mathbb{R}$ (univariate approximation) such that $\Omega_x$ is partially ordered by R. Therefore corresponding to each $x \in \Omega_x$, we can define a function $K^s(\cdot,x)$ and associate a weight $\theta^x$ with it, making weights a function of the continuous variable x ($\theta^x$ is used in the penalty definition in (\ref{norm1}) and (\ref{norm2})). Now considering the discrete case and  applying the same ordering R on  $\Theta^s_{qr} =\{ \theta^{x_i} | x_i \in X_s\}$, represented as $\Theta^{s}_{x} = Pe^s_x \Theta^s_{qr}$. Here $Pe^s_x$ is the permutation operator at scale s in the x-direction (enforcing relation R) and $\Theta^s_{qr}$ is the set of coordinates for the bases set spanning the approximation space $\Gamma^s$. The initial ordering of $\theta^{x_i} \in \Theta^s_{qr}$ is determined by the ordering of the corresponding basis functions in the bases set. In the current research we implement the penalization of sharp changes by constraining the behavior of basis functions at data points ($x_i$) in close proximity (as per the ordering induced by R) to vary in a smooth manner. This is achieved by constraining the rate of change of the weights of these basis functions. Thus for a univariate function $\tilde{f} = B^s \Theta^s_{qr}$ (where $B^s$ is the spanning basis for $\Gamma^s$), we consider the following proxies for the first and second order derivative based penalties.

\begin{equation}\label{norm1}
\zeta{(\tilde{f})}_{q =1} = \bigintsss_\Omega_x \Bigg[\frac{d \theta^{x}}{dx}\Bigg]^2dx \approx ||D^1 \Theta^{s}_{x}||^2 = {\Theta^s_{qr}}^T {Pe^s_x}^T {D^1}^T{D^1}Pe^s_x \Theta^s_{qr}
\end{equation}

\begin{equation}\label{norm2}
\zeta{(\tilde{f})}_{q =2} =  \bigintsss_\Omega_x \Bigg[\frac{d^2 \theta^{x}}{dx^2}\Bigg]^2dx \approx ||D^2 \Theta^{s}_{x}||^2 = {\Theta^s_{qr}}^T {Pe^s_x}^T {D^2}^T{D^2} Pe^s_x \Theta^s_{qr}
\end{equation}

Here $D^q$ is a difference operator of order q on $\Theta_x^{s}$. Beginning with the difference operator for individual $\theta_i \in \Theta_x^{s}$ (represented as $\Delta^q$ for $q^{th}$ order penalty) we have

\[\Delta^1 \theta_i = \theta_i - \theta_{i-1}\]
\[\Delta^2 \theta_i  = \Delta^1(\Delta^1 \theta_i) = \theta_i - 2 \theta_{i-1} + \theta_{i-2}\] 
\[\vdots\]
\[\Delta^q \theta_i = \Delta^1 (\Delta^{q-1} \theta_i ) \]

And in matrix form, $\Delta^q$ represented as $D^q$ can be expressed as follows (considering 5 basis functions and q = 1, 2 respectively as example)

\[D^1 = \begin{pmatrix} -1 &\phantom{-}1& \phantom{-}0& \phantom{-}0& \phantom{-}0 \phantom{-}\\\phantom{-}0 & -1 & \phantom{-}1 & \phantom{-}0 & \phantom{-}0 \phantom{-}\\  \phantom{-}0 & \phantom{-}0 & -1 & \phantom{-}1 & \phantom{-}0 \phantom{-}\\ \phantom{-}0& \phantom{-}0& \phantom{-}0& -1 &\phantom{-}1\phantom{-} \end{pmatrix} \quad D^2 = \begin{pmatrix}\phantom{-}1& -2 &\phantom{-}1 &\phantom{-}0 &\phantom{-}0 \phantom{-}\\ \phantom{-}0& \phantom{-}1& -2& \phantom{-}1& \phantom{-}0\phantom{-}\\ \phantom{-}0& \phantom{-}0 &\phantom{-}1 &-2& \phantom{-}1\phantom{-} \end{pmatrix}\]

Based on the requirement, it is straightforward to come up difference operators for higher order penalties ($D^q$ for $q > 2$). It should be noted that (\ref{norm1}) and (\ref{norm2}) indeed define a seminorm on the space $\mathcal{H}_s$, again confirming the fact that these norms are evaluated in some subspace of $\mathcal{H}_s$ (just penalizing the projection in the subspace $\mathcal{H}_{s,1}$).

Coming back to problem (\ref{eq5}), the loss function and the stabilizing operator can be represented as 

\begin{equation}\label{eq6}
V = ||Y - \mathcal{T}^s\tilde{f}||^2_2 =   ||Y -B^s \Theta^s_{qr}||^2_2
\end{equation}

\begin{equation}\label{eq8}
\zeta(\tilde{f})_q = {\Theta^s_{qr}}^T {Pe^s_x}^T {D^q}^T{D^q} Pe^s_x \Theta^s_{qr}
\end{equation}

Now putting (\ref{eq6}) and (\ref{eq8}) in (\ref{eq5}) leads to the following modified formulation for univariate approximations
 
 \begin{equation}\label{eq9}
\operatorname*{\min_{\substack{\Theta^s_{qr} \in {\R^{|X_s|}} }} } \Bigg[ \frac{1}{n} ||Y -B^s \Theta^s_{qr}||^2_2 + \lambda_s {\Theta^s_{qr}}^T {Pe^s_x}^T {D^q}^T{D^q} Pe^s_x \Theta^s_{qr} \Bigg]\
 \end{equation}
 
\begin{figure}
\centering
\includegraphics[width=4.5in]{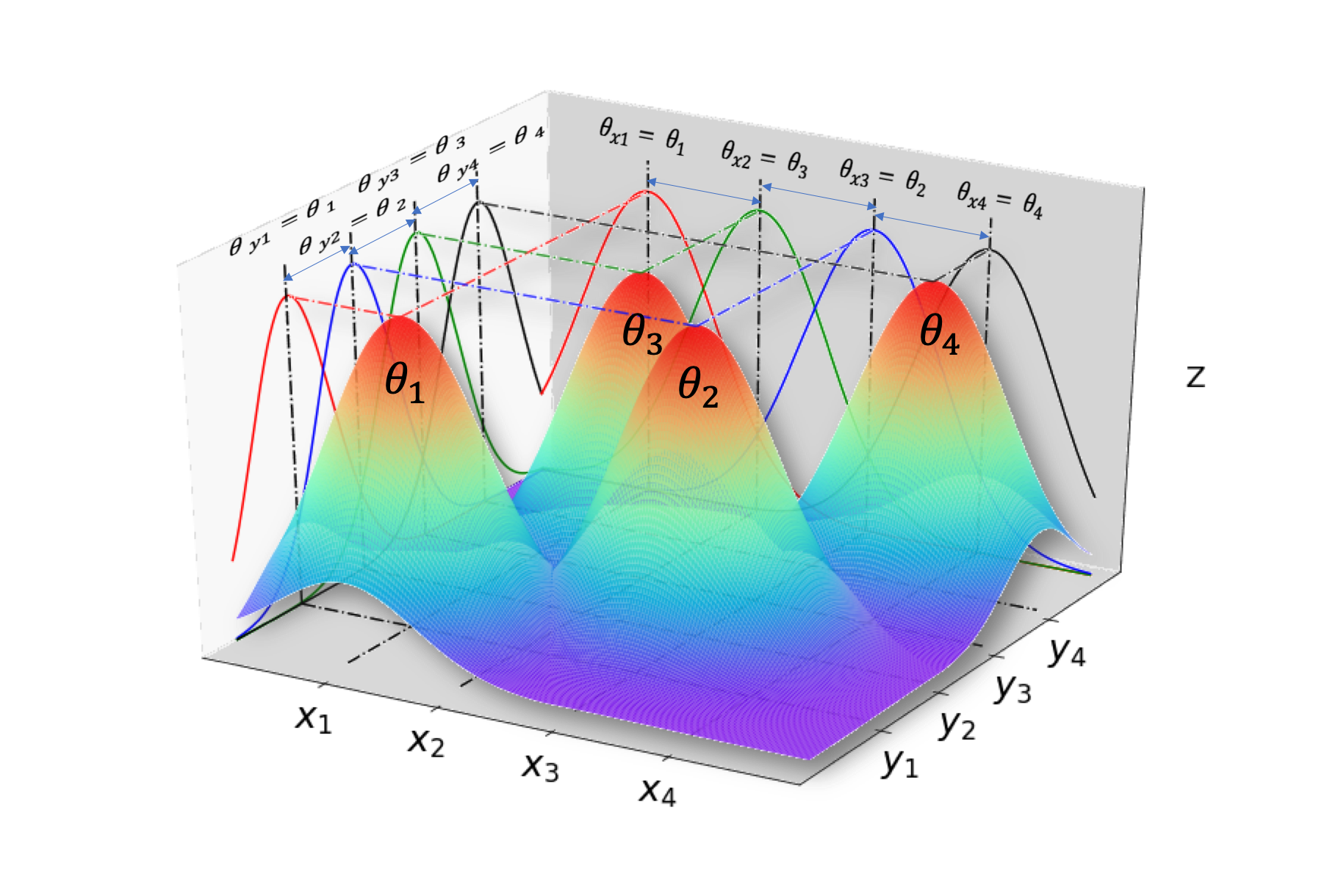}
\caption{Nature of penalty for 2-D basis functions imposed by projection on the corresponding dimensions and application of a permutation operator}
\label{img1}
\end{figure}

For modeling in higher dimensions, we put independent penalties in each dimension in a similar way as before. Let $\Theta^{s}_{i}$ is the ordering of the weight vector as per the Relation $\leq$ on coordinates in the $i^{th}$ dimension and $Pe^s_i$ is the corresponding permutation operator which transforms $\Theta^s_{qr}$ ($\Theta_{i}^{s} = Pe^s_{i}\Theta^s_{qr}$). Also let $Q = [q_1,q_2,...q_d]$ be the vector of order of penalties across each of the dimensions (for $\R^d$) with $\Lambda_s = [\lambda_s^1, \lambda_s^2,...\lambda_s^d]$ being the set of corresponding regularization parameters. Therefore multidimensional penalty operator ($\mathcal{P}^Q_s $) has the representation

\begin{equation}
    \mathcal{P}^Q_s = \sum_{i = 1}^d \lambda^i_s \Psi_s^{q_i} \quad \text{ where }\quad \Psi_s^{q_i} =  {Pe^s_i}^T {D^{q_i}}^T{D^{q_i}} Pe^s_i
    \label{ppenalty}
\end{equation}

For illustrating the penalty structure, we have presented a test case in Figure \ref{img1}. Here we have the X-Y plane as the approximation domain. Assuming at any scale s, $\Theta^s_{qr} = [\theta_1, \theta_2, \theta_3, \theta_4]$. Depending on the location of these basis function (in the data space), we have the following permutation operators

\[Pe_x^s = \begin{bmatrix} 1 & 0 & 0& 0 \\0 & 0 & 1 & 0 \\  0 & 1 & 0 & 0 \\ 0& 0& 0& 1 \end{bmatrix} \quad Pe_y^s = \begin{bmatrix} 1 & 0 & 0& 0 \\0 & 1 & 0 & 0 \\  0 & 0 & 1 & 0 \\ 0& 0& 0& 1 \end{bmatrix}
\]
giving us $\Theta_x^{s} = [\theta_1, \theta_3, \theta_2, \theta_4]$ and $\Theta_y^{s} = [\theta_1, \theta_2, \theta_3, \theta_4]$ (coefficients according to the ordering R ($\leq$) as described before).

Hence we have the following analogous problem formulation to (\ref{eq5}) for the L-2 regularization network in higher dimensions

\begin{equation}\label{l2high}
A_sf =  \operatorname*{arg\min_{\tilde{f} \in {\Gamma^s}}} \Bigg[ \frac{1}{n} ||Y - \mathcal{T}^s \tilde{f}||^2_2 + \sum_{i=i}^d\lambda^i_s ||J^i_s \tilde{f}||^2_{\mathcal{H}_s}  \Bigg]\
\end{equation}

and rewriting it with basis and penalty operators, we obtain the following regularization network problem.

\begin{algorithm}
\caption{Main Algorithm}
\label{multiscale_algo}
\begin{algorithmic}[1]
\STATE{\textbf{INPUT}:\\
\quad \textit{Parameters}: $(T>0, M > 1) \in \R^2$\\
\quad \textit{Dataset}: $D = \{(x_i,y_i) \in \Omega_x \times  \Omega_y\}_{i=1}^n$ \\
\quad \textit{Prediction points}: $X_m \subset \Omega_x$}
\STATE{\textbf{OUTPUT}:\\ 
\quad \textit{Convergence length scale}: $\epsilon_t \in \R$\\
\quad \textit{Sparse model}: $X_{t} \subseteq X,C_{t} \in \R^{|X_{t}|}$\\
\quad \textit{Sparse representation}: $D_t = (X_{t},Y_t) \subseteq D$\\
\quad \textit{Predictions at $X_m$}: $P_m \in \R^{m}, {std_m} \in \R^{m}$}\\
\noindent\rule{11.9cm}{0.4pt}
\STATE{Initialize: $s =0, l_s = 0, T_h = 0$}
\WHILE{{$l_s < n$}}
\STATE{Compute covariance kernel: $G_s$ on $X$ with $\epsilon_s = T/M^s$}
\STATE{Update numerical rank for current scale: $l_s = rank(G_s)$}
\STATE{Remove sampling bias: $W = AG_{s}$ with $A = [a_{i,j}] \in \R^{k \times n}$( $a_{i,j} \sim \mathcal{N}(0,1))$}
\STATE{Generate permutation information: $WP_{qr} = QR$}
\STATE{Produce sparse representation and corresponding bases: $(X_s,Y_s)$ and $B^s$}
\STATE{$[\hat{\Lambda}_{s},\hat{Q}, Cost_s] \leftarrow GCV\_model\_evaluate(B^s,D)$} (illustrated in \ref{costgcv})
\STATE{Compute the optimal weights: $\hat{\Theta}_{qr}^s$ from (\ref{coorfinal})}
\STATE{\textbf{if} $s == 0\ \textbf{or}\ Cost_s < T_h$ : }
\STATE{\quad $[t,\epsilon_t,X_{t},Y_t,C_t,\Lambda_{t},Q_t,T_h]  \leftarrow [s,\epsilon_s,X_s,Y_s,\hat{\Theta}_{qr}^s,\hat{\Lambda}_{s},\hat{Q},Cost_s]$}
\STATE{Update scale: $s=s+1$}
\ENDWHILE
\STATE{$P_m \leftarrow Predict\_mean (\epsilon_t,X_{t},C_{t},X_m)$ (Algorithm \ref{multiscale_algo2})}
\STATE{$std_m \leftarrow Predict\_CI (\epsilon_t,X_{t},C_t,X_m,\Lambda_{t},Q_t,D)$ (Algorithm \ref{multiscale_algo3})}
\STATE{\textbf{return} [$\epsilon_t, (X_{t},C_{t}),(X_t,Y_t),P_m,std_m$]}
\end{algorithmic}
\end{algorithm}

\begin{equation}\label{highd}
\operatorname*{\min_{\substack{\Theta^s_{qr} \in {\R^{|X_s|}} }} } \Bigg[ \frac{1}{n} ||Y -B^s \Theta^s_{qr}||^2_2 +  {\Theta^s_{qr}}^T \mathcal{P}^Q_s{\Theta^s_{qr}}\Bigg]
\end{equation}

\subsection{Fitting the regularization network at multiple scales}
The theory of regularization networks has been developed closely in relation to the Vapnik's ideas on statistical learning theory \cite{vapnik2013nature}. If we have a finite set of training data, then the approximation has to be constrained to a small hypothesis space ($\Gamma^s$). This concept has been formalized through the capacity of a set and controlling its capacity for proper generalizable approximations. This implementation of \textit{capacity control} exactly corresponds to finding the optimal $\Lambda_s$ for a justified trade-off. In this research, we implement and analyze the performance of Generalized Cross-Validation ($GCV$) for evaluating the performance (quality) of the model at a particular scale. The scale with the minimum optimized GCV metric is regarded as the convergence scale \cite{ruppert2003semiparametric} and the corresponding regularization network is declared as the winner and the most suitable for modeling the given dataset D

Working with the regularization problem (\ref{highd}), if we differentiate the cost function with respect to $\Theta^s_{qr}$, we obtain the normal equations

 \begin{equation}\label{normmm}
 \Bigg[\frac{1}{n} {B^s}^T B^s +\mathcal{P}^Q_s\Bigg]\hat{\Theta}^s_{qr} = \frac{1}{n} {B^s}^T Y
 \end{equation}
 
giving us the $\hat{\Theta}^s_{qr}$ as a function of hyperparameters $\Lambda_s = [\lambda_s^1, \lambda_s^2,...\lambda_s^d]$

 \begin{equation}\label{coorfinal}
\hat{\Theta}^s_{qr}(\hat{\Lambda}_s) =   \Big[{B^s}^T B^s + n \widehat{\mathcal{P}^Q_s}\Big]^{-1}{B^s}^T Y
 \end{equation}
 
Here $\widehat{\mathcal{P}^Q_s}$ represents the estimated penalty operator $\mathcal{P}^Q_s$ (\ref{ppenalty}) after substituting optimal hyperparameters $\Lambda_s$(represented as $\hat{\Lambda_s}$). Therefore, the whole objective of model fitting on the dataset reduces  to choosing the right $\Lambda_s$(hyperparameters quantifying regularization along each dimension). Moving forward, we discuss the main algorithm which precisely does this for all the competing, scale dependent regularization networks and chooses the one with the highest generalizable performance. If two scales have the same model fitting cost, then the one with less complexity is chosen (less number of data points in the sparse set $X_s$).

Our approach (Algorithm \ref{multiscale_algo}), takes a dataset, where a data point $x_i \in \R^d$ is mapped to an observed value $y_i \in \R$.  In matrix form $Y= (y_1,y_2,....y_n)$ values are obtained at data points $X = \{x_1, x_2,....,x_n\}$ ($Y \in \R^n$ and $X \in \R^{n \times d}$). The scalars $[T, M]\in \R^2$ are the algorithmic hyperparameters defined by the user. These choices inform the structure of the positive definite function ($K: \R^d \times \R^d \to \R$) used in the algorithm. Here we work with the squared exponential kernel (\ref{sqexp}) \cite{rasmussen2004gaussian} for mapping the covariance structure and generating the space of trial functions $\Gamma^s$ (\ref{eq4}) at each scale s. 

\begin{equation}
G_s (x_i,x_j) = \exp\left({-\frac{||x_i-x_j ||^2}{\epsilon_s}}\right) \mbox{ ;  }  \epsilon_s=\frac{T}{M^s} 
\label{sqexp}
\end{equation}

Here $\epsilon_s$ is the length scale parameter determining the  support of the basis set at scale s. M is assumed to be 2 (Based on \cite{bermanis2013multiscale}). This choice of M reduces the length scale of the kernel
($G_s$) by a factor of 0.5 at each scale increment, providing an intuitive understanding of how
the support of basis functions is adapted to scale variation. Furthermore, if we assume the diameter of the dataset to be distance between the most distant pair of datapoints, then T is given by

\begin{equation}
T = 2 (Diameter(X)/2)^2
\end{equation}

Besides these parameters, the algorithm also accepts $X_m = (x_i,x_2,...,x_m) \subset \Omega_x \subset \R^{d}$, which represent the data points at which the user wants to predict the underlying function.

In this section we explain how we infer the convergence scale ($t$) and the sparse set $X_t$. The final prediction at the convergence scale will be explained in detail in the following section. Given the Dataset $D = \{(x_i, y_i) \in \Omega_x \times \Omega_y\}_{i=1}^n$, Algorithm \ref{multiscale_algo} begins with the computation of the covariance operator $G_s$ (\ref{sqexp}). However, based on research such as \cite{de2010stability,fasshauer2009preconditioning} , the distribution of the dataset might lead to ill-conditioning of this covariance kernel. Therefore we carry out a column pivoted QR decomposition to identify the space $\Gamma^s$ (at each scale) which approximates the span of the trial functions $K^s(\cdot, x_j), [1 \leq j \leq n]$ at scale s (\ref{eq4}). The QR decomposition is carried out on W (instead of $G_s$ directly) for obtaining the Permutation matrix $P_{qr}$. $W$ is produced by the product of a random normal matrix $A$ with the $G_s$. Here we have $A \in \R^{k \times n}$ with $(l_s = rank(G_s)) \leq k \leq n$. For our experiments we have assumed $k = l_s+8$ (as in \cite{bermanis2013multiscale}), which means we sample 8 additonal rows to account for numerical round-offs during the QR decomposition. The permutation matrix $P_{qr}$ produced by the decomposition captures the  information content of each column of W. $P_{qr}$ is then used to extract independent columns with the biggest norm contributions (forming the bases set $B^s$) along with the observation points ($X_s$) these columns correspond to in the covariance kernel $G_s$. The ordering of basis functions in $B^s$ (governed by $P_{qr}$ and representing the information content in decreasing order) determine the ordering of $\theta^{x_i} \in \Theta^s_{qr}$ (here $x_i \in X_s$). The dimension of the bases comes from the numerical rank ($l_s$) of $G_s$ estimated by strategies such as a $Rank\ Revealing-QR$ or a $SVD$ decomposition. Finally $GCV\_model\_evaluate$ subroutine is called which fits the regularization network at the current scale. In essence we follow the ideas from \cite{wahba1990spline} for solving a penalized objective of the form (\ref{highd}), and thus minimize the Generalized Cross Validation metric which is given as

\begin{equation}
GCV_s (\Lambda_s) = \frac{1}{n} ||(I - U(\Lambda_s))Y||^2 / \Big[\frac{1}{n}Tr(I - U(\Lambda_s)) \Big]^2
\end{equation}

where $U(\Lambda_s)$ is the influence matrix satisfying

\begin{equation}
\mathcal{T}^s(A_sf) = U(\Lambda_s)Y 
\end{equation}

\begin{equation}
\label{ulam}
   \text{Thus   }\quad  U(\Lambda_s) = B^s\Big[{B^s}^T B^s + n \mathcal{P}^Q_s\Big]^{-1}{B^s}^T 
\end{equation}

Here the objective is to find the optimal penalty vector $\Lambda_s$. However, besides the regularization parameters ($\lambda^i_s$), we also have to find a suitable penalty order across each dimension $Q = [q_1, q_2,...,q_d]$. So, for every dimension $i$, we just consider $q_i$ = 1 and 2 (higher order penalties were found to oversmooth approximations weakening the local structure), and choose the final penalty vector Q (composed of either $1^{st}$ or $2^{nd}$ order penalties across each dimension), that lead to a overall smallest $GCV_s(\Lambda_s)$. Hence, in essence we are solving the following formulation:

\begin{equation}\label{costgcv}
Cost_s =\operatorname*{ \min_{\substack{{\Lambda_s>0}\\{Q|q_i \in \{1,2\}}}}} GCV_s, \text{  with  } [\hat{\Lambda}_{s},\hat{Q}] = arg\operatorname*{ \min_{\substack{{\Lambda_s>0}\\{Q|q_i \in \{1,2\}}}}}  GCV_s
\end{equation}
$GCV\_model\_evaluate$ from Algorithm \ref{multiscale_algo} implements this optimization problem. Here $\Lambda_s >0$ refers to $\lambda_s^i >0\  \forall i$

Therefore, when  Algorithm \ref{multiscale_algo} exits the $while$ loop (after covariance kernel becomes numerically full rank), we obtain the convergence scale $t$ (the scale with the minimum $Cost_s$ (\ref{costgcv})), the sparse set $X_{t}$ and corresponding coordinate of projection $C_{t}$ ($C_{t}$ is same as $\Theta^s_{qr}$ at optimal scale $s = t$  in  (\ref{coorfinal})). \textbf{Thus, we have the sparse representation $D_t = (X_t,Y_t)$ and the sparse model $(X_t,C_t)$ for dataset D}. 

One additional thing to discuss  in Algorithm \ref{multiscale_algo} (before we move on to the $Predict\_mean()$ and $Predict\_CI()$ functions in Algorithm \ref{multiscale_algo2} and \ref{multiscale_algo3} respectively) is the termination condition for the $while$ loop. For that we provide the following result

\begin{theorem}
The number of while loop iterations for Algorithm \ref{multiscale_algo} are finite and grow with data size n at $\mathcal{O}(log_2(n))$
\end{theorem}

\begin{proof}
Following the work of \cite{bermanis2013multiscale} , if $\phi$ represents the precision of rank for the Gaussian kernel matrix, then we can define its numerical rank as

\begin{equation}
 l_s^\phi (G_s) = \#\Bigg(j: \frac{\sigma_j (G_s)}{\sigma_0 (G_s)} \geq \phi \Bigg)
\end{equation}

where $\sigma_j(G_s)$ is the $j^{th}$ largest singular value of $G_s$. Also if we assume $|V_i|$ represents the length of the bounding box of the data in $i^{th}$ ($i \in [1,d]$) dimension, then given the length scale parameter $\epsilon_s$, the rank of the Gaussian kernel can be bounded above as 

\begin{equation}
l_s^\phi (G_s) \leq \prod_{i = 1}^d \Bigg( \frac{2 |V_i|}{\pi} \sqrt{\epsilon_s^{-1}ln(\phi^{-1})} + 1 \Bigg)
\end{equation} 

Then using proposition 3.7 in \cite{bermanis2013multiscale}, we recall the fact that numerical rank of the gaussian kernel matrix is proportional to the volume of the minimum bounding box $Vol = V_1 \times V_2 \times ....\times V_d$ and to  $\epsilon_s^{-d/2}$ . Therefore for a fixed data distribution, following relation holds

\begin{equation}\label{r_del}
 l_s^\phi (G_s) \propto \epsilon_s^{-d/2}  \propto 2^{sd}
\end{equation}

Hence numerical rank ($l_s$) of $G_s$ increases exponentially with scale $s$ until it becomes full rank ($l_s = n$). The result directly follows from here also establishing the finiteness of the while loop. 
\end{proof}

\begin{algorithm}
\caption{$Predict\_mean (\epsilon_t,X_{t},C_{t},X_m)$}
\label{multiscale_algo2}
\begin{algorithmic}[1]

\STATE{\textbf{INPUT}:\\
\quad \textit{Length scale parameter}: $\epsilon_t \in \R$\\
\quad \textit{Sparse model}: $(X_t,C_t)$ \\
\quad \textit{Prediction points}: $X_m \subset \Omega_x$}
\STATE{\textbf{OUTPUT}:\\ 
\quad \textit{Prediction at $X_m$}: $P_m \in \R^m$}\\
\noindent\rule{11.9cm}{0.4pt}
\STATE{Compute prediction bases: $B^t_m$ for $X_{t}$ and $X_m$ (using $\epsilon_t$ (\ref{sqexp}))}
\STATE{Compute mean prediction: $P_m  = B^t_m C_{t}$ (\ref{C_c})}
\STATE{\textbf{return} $P_m$}
\end{algorithmic}
\end{algorithm}

\subsection{Inference at convergence scale ($t$)} Algorithm \ref{multiscale_algo} defined the steps for obtaining the convergence scale $t$, the sparse subset $X_{t}$ and corresponding coordinate of projection $C_{t}$ (within $While$ loop) for modeling the dataset $D$. However, given a proper approximation space ($\Gamma^{t}$ spanned by bases centered at the sparse set $X_{t}$), the second step in modeling is always to generalize this inference over the entire domain. Hence, we use the obtained sparse model ($X_t,C_t$) to make inference at new data points of interest (Algorithm \ref{multiscale_algo2} and \ref{multiscale_algo3}).

\begin{algorithm}
\caption{$Predict\_CI (\epsilon_t,X_{t},C_t,X_m,\Lambda_{t},Q_t,D)$}
\label{multiscale_algo3}
\begin{algorithmic}[1]

\STATE{\textbf{INPUT}:\\
\quad \textit{Length scale parameter}: $\epsilon_t \in \R$\\
\quad \textit{Sparse model}: $(X_t,C_t)$ \\
\quad \textit{Prediction points}: $X_m \subset \Omega_x$
\\
\quad \textit{Hyperparameters}: $\Lambda_t \in \R^d,Q_t \in \R^d$
\\
\quad \textit{Data}: $D$
}
\STATE{\textbf{OUTPUT}:\\ 
\quad \textit{Confidence Interval for prediction at $X_m$}: $std_m \in \R^m$}\\
\noindent\rule{11.9cm}{0.4pt}
\STATE{Compute data bases: $B^t$ for $X_t$ and $X$ (using $\epsilon_t$ (\ref{sqexp}))}
\STATE{Compute $U(\Lambda_t)$: from (\ref{ulam}) using $B^t$, $\Lambda_t$ and $Q_t$}
\STATE{Compute $\mathcal{T}^{t}({A_{t}f}) $: $B^tC_t$}
\STATE{Compute $\hat{\sigma}^2_{\epsilon}$: substitute $Y,\mathcal{T}^{t}({A_{t}f}),U(\Lambda_t)$ in  (\ref{sig})}
\STATE{Compute prediction bases: $B^t_m$ for $X_{t}$ and $X_m$ (using $\epsilon_t$)}
\STATE{Compute the interval ($std_m$): substitute computed quantities in (\ref{std})}
\STATE{\textbf{return} $std_m$}
\end{algorithmic}
\end{algorithm}

Starting with the procedure for getting predictions at data points $X_m$ defined in $Predict\_mean()$ - shown as Algorithm \ref{multiscale_algo2}, we formulate the set of bases centered at the sparse set $X_{t}$ with respect to the prediction location $X_m$ (represented as $B^t_m$), giving the following representation for approximation of the underlying process $f$ restricted to the set $X_m$

\begin{equation}
\label{C_c}
P_m = {A_{t}f}|_{X_m} = B^t_mC_{t}
\end{equation}

where $C_{t}$ (referred to as coordinate of projection) is obtained from Algorithm \ref{multiscale_algo}. It is   crucial to note here, that for producing these approximations, we just needed the sparse model - $(X_t,C_t)$. We don't need access to the full dataset D. This characteristic of the approach can lead to massive storage and computational savings.

Again, following the ideas of \cite{wahba1990spline}, we have presented the steps for getting the confidence intervals (CI) in Algorithm \ref{multiscale_algo3} ($Predict\_CI()$). Here we use an empirical unbiased estimate of $\sigma^2_{\varepsilon}$ for these confidence bounds.
\begin{equation}
\label{sig}
\hat{\sigma}^2_{\varepsilon} = \frac{{||Y - \mathcal{T}^{t}({A_{t}f})||}^2_2}{df_{res}}
\end{equation}

Here $df_{res}$ represents the degree of freedom for the residual for which we use the non-parametric estimate $df_{res}  = n - 2\cdot tr(U(\Lambda_t)) + tr(U(\Lambda_t) \cdot U(\Lambda_t)^T)$ (with $U(\Lambda_t)$ as defined in (\ref{ulam}) at $s = t$). Here $tr$ is the trace operator. Following the recommendation of \cite{ruppert2003semiparametric}, the standard deviation for the error term could be estimated as 
\begin{align}\label{std}
\widehat{std_m}((A_{t}f)(x^*) - f(x^*))
&= \hat{\sigma_{\varepsilon}}\sqrt{ B^t(x^*) \Big[{B^{t}}^T B^{t} + n \widehat{\mathcal{P}^Q_{t}}\Big]^{-1} {B^t(x^*)}^T}
\end{align}

Now, it is straightforward to state that $100(1-\alpha_c) \%$ confidence intervals will be written as 
\begin{equation}
\label{ci}
{A_{t}f}(x^*) \pm t \Big(1 - \frac{\alpha_c}{2}; df_{res} \Big) \widehat{std_m}(({A_{t}f})(x^*) - f(x^*))
\end{equation}

Unlike mean approximation $A_tf|_{X_m}$, unfortunately, if we want to augment our predictions at new data points $X_m$ with confidence bounds, then we need to go back to the full dataset D. This is because in (\ref{std}), we need to compute $B^{t} \in \mathbb{R}^{|X| \times |X_{t}|}$ which involves full data $X$.

\section{Approximation properties}
For developing the results in this section, we have taken ideas from \cite{iske2011scattered,hsing2015theoretical,ferraty2006nonparametric,de1978practical}.  Here many of the proofs developed consider $Y \in Dom(\mathcal{T}^{\dagger})$ with $Y \in Ker(\mathcal{T}^{\dagger})$ as a special case. Our first main result provides an inner product representation for the approximation $A_sf$ to $f$, produced at scale s. This alternate representation will help us with a more precise consistency and error analysis. Defining $\delta_x$ as the evaluational functional for $f$, i.e.  $\delta_x(f) = f(x)$ gives us the dual space
%
$\mathcal{H}^{*}_s = \Bigg\{ \sum_{x_j \in X} c_j \delta^s_{x_j} \Bigg\}$,
and by assuming the traditional definition of norm in this dual space, we have

\begin{equation}
K^s(x,y) = <K^s(x,\cdot),K^s(y,\cdot)>_ {\mathcal{H}_s} = <\delta^s_x,\delta^s_y>_ {\mathcal{H}_s^*}  \quad x,y \in \Omega
\end{equation}

\begin{definition}
Let the pointwise error functional at any data point $x \in \Omega_x$ has a representation
\begin{equation}\label{erroreq1}
E^x_{\Lambda_s} = \delta^s_x - M^T_{\Lambda_s}(x)  \delta^s_X
\end{equation}
Here $\Lambda_s = [\lambda_s^1,...,\lambda_s^d]$ is the optimal set of regularization parameters in d-dimensions. $\delta^s_x$ is identified as the Riesz representation of the evaluation functional at x in the dual space of $\mathcal{H}_s$ and $M_{\Lambda_s}$ represented as $M_{\Lambda_s}(x) = [M^1_{\Lambda_s}(x),M^2_{\Lambda_s}(x),...,M^n_{\Lambda_s}(x)] \in \R^n$, is a set of n appropriate functions ($M^j_{\Lambda_s}$ depends on $x_j \in X$) evaluated at $x \in X$.
Then, given such a representation, we denote the magnitude of expected pointwise approximation error as
\begin{equation}\label{pointwiseerror}
Error(x) = |f(x) - \mathbb{E}[A_sf](x)| = |\delta^s_x(f) - M^T_{\Lambda_s}(x)\delta^s_X (f)| = |E^x_{\Lambda_s}(f)|
\end{equation}
\label{def43}
\end{definition}

Hence from this definition, with some appropriate set of n functions $\{M^j_{\Lambda_s}\}$, evaluated at $x \in \Omega_x$, we have $\mathbb{E}[A_sf](x) = M^T_{\Lambda_s}(x)\delta^s_X (f)$.

Moving further, we again define a semi-norm which relates the penalty in multiple dimensions (denoted by $\mathcal{P}^Q_s$ (\ref{ppenalty})) to the behavior of the basis functions $B^s$ spanning the approximation space. In essence, this formalizes constraining of the approximation space to limit its capacity.

\begin{definition}
For bases $B^s$ at scale s, we define a semi-inner product and the corresponding semi-norm in n-dimensional Euclidean space as 

\begin{equation}\label{seminorm1}
\Big<a,b\Big>_{T_s} = a^TT_sb \quad ||a||_{T_s} = \Big<a,a\Big>^{1/2}_{T_s}\text{, } a,b \in \R^n
\end{equation}

where $T_s$ is a self adjoint operator satisfying the relation
\begin{equation}\label{relation}
P_{qr}|_{X_s}T_sB^s = n \mathcal{P}^Q_s
\end{equation}

Here $P_{qr}|_{X_s} =  \begin{bmatrix}
I_{|X_s|}\ |\  0
\end{bmatrix}P_{qr}$, with $P_{qr}$ being the permutation operator for column pivoted QR in Algorithm \ref{multiscale_algo}, and $I_{|X_s|}$ is a $|X_s|$-dimensional identity matrix. $\mathcal{P}^Q_s$ is the total penalty operator in multiple dimensions.

\end{definition}

Now, with the representation of error functional as in (\ref{erroreq1}), we state the following result.

\begin{theorem}
\label{thcost}
The solution $\hat{M}_{\Lambda_s}(x)$ to the penalized error minimization problem
\begin{equation}\label{errfunc}
 \hat{M}_{\Lambda_s}(x) = \arg \min_{M_{\Lambda_s}(x) \in \R^n} \Big[||E^x_{\Lambda_s}||^2_{\mathcal{H}^{*}_s} + || M_{\Lambda_s}(x)||^2_{T_s} \Big]
\end{equation}
satisfies the inner product representations $<Y,\hat{M}_{\Lambda_s}(x)> = (A_sf)(x)$ for $A_sf$ and $<\mathcal{T}^sf,\hat{M}_{\Lambda_s}(x)> = \mathbb{E}[A_sf](x)$ for mean approximation $\mathbb{E}[A_sf]$ at any $x \in \Omega_x$.
\end{theorem}

\begin{proof}
Starting with the error functional norm
\begin{align*}
||E^x_{\Lambda_s}||^2_{\mathcal{H}^{*}_s}  &= < \delta^s_x - M^T_{\Lambda_s}(x)	 \delta^s_X,  \delta^s_x - M^T_{\Lambda_s}(x) \delta^s_X>_{\mathcal{H}^{*}_s}\\
&= ||\delta^s_x||^2_{\mathcal{H}^{*}_s} - 2M^T_{\Lambda_s}(x) \delta^s_X\delta^s_x + M^T_{\Lambda_s}(x) \delta^s_X{\delta^s_X}^T  M_{\Lambda_s}(x)
\end{align*}

Therefore the quantity to be minimized from (\ref{errfunc}) can be written as
\begin{equation}
 ||\delta^s_x||^2_{{\mathcal{H}}^{*}_s} - 2M^T_{\Lambda_s}(x) \delta^s_X\delta^s_x + M^T_{\Lambda_s}(x) \delta^s_X{\delta^s_X}^T  M_{\Lambda_s}(x)+
{M^T_{\Lambda_s}(x)}  T_s  M_{\Lambda_s}(x)
\label{costmod}
\end{equation}

Now, based on the property of dual space, we know at scale s, 

\[
<\delta^s_a,\delta^s_b>_{\mathcal{H}^{*}_s}  = K^s(a,b) 
\]

Also, let $R_s(x) = \delta^s_X \delta^s_x = (K^s(x,x_1), K^s(x,x_2),....,K^s(x,x_n)) \in \R^{n}$ and $G_s = \delta^s_X {\delta^s_X}^T$ . 
Now, differentiating (\ref{costmod}) with respect to $M_{\Lambda_s}(x)$ and setting it to 0 gives
\begin{equation}\label{syseq}
R_s(x) =G_s  M_{\Lambda_s} +T_s   M_{\Lambda_s}
\end{equation}

Now, since $G_s$ has a rank of $l_s$ at scale s which is also true for orthogonal projection operator for $B^s$ (given as $ B^s({B^s}^TB^s)^{-1}{B^s}^T$). Therefore in order to sample independent equations from the system (\ref{syseq}), we use the same method as in Algorithm \ref{multiscale_algo}. We again create the matrix $W( = AG_s)$ and carry out a column pivoted QR decomposition $WP_{qr} = QR $. Now applying the permutation operator $P_{qr}$ on system (\ref{syseq}) and sampling the first $l_s$ equation.
\[
P_{qr}R_s(x) = P_{qr} G_s  M_{\Lambda_s} +P_{qr}T_s  M_{\Lambda_s}(x)
\]

For sampling first $|X_s|$ (the cardinality of the sparse set $X_s$ is $l_s$) equations and to remove redundancy, pre-multiplying by $\begin{bmatrix}
I_{|X_s|}\ |\  0
\end{bmatrix}$
\[
R_s(x)|_{X_s} = {B^s}^T   M_{\Lambda_s} + 
\begin{bmatrix}
I_{|X_s|}\ |\  0
\end{bmatrix}
P_{qr}T_s    M_{\Lambda_s}(x)
\]
Using the relation from (\ref{relation})

\[
\begin{bmatrix}
I_{|X_s|}\ |\  0\end{bmatrix} P_{qr}T_sB^s = n\mathcal{P}^Q_s({B^s}^TB^s)^{-1}{B^s}^TB^s\]
\[
 \implies \begin{bmatrix}I_{|X_s|}\ |\  0\end{bmatrix} P_{qr}T_s = n\mathcal{P}^Q_s({B^s}^TB^s)^{-1}{B^s}^T
\]

Putting it back, we get
\begin{align*}
R_s(x)|_{X_s} &= {B^s}^T   M_{\Lambda_s} + 
n\mathcal{P}^Q_s({B^s}^TB^s)^{-1}{B^s}^T  M_{\Lambda_s}(x)
\\
&=
( {B^s}^TB^s + n \mathcal{P}^Q_s)({B^s}^TB^s)^{-1}{B^s}^{T}  M_{\Lambda_s}(x)
\end{align*}
Therefore, $
{B^s}^{T}   M_{\Lambda_s}(x) = ({B^s}^TB^s)( {B^s}^TB^s + n \mathcal{P}^Q_s)^{-1}R_s(x)|_{X_s}
$
\[ \implies
 \hat{M}_{\Lambda_s}(x) =  B^s( {B^s}^TB^s + n \mathcal{P}^Q_s)^{-1}R_s(x)|_{X_s}
\]
Hence,
\begin{align*}
<\mathcal{T}^sf, \hat{M}_{\Lambda_s}(x)> &= <\mathcal{T}^sf,B^s( {B^s}^TB^s + n \mathcal{P}^Q_s)^{-1}R_s(x)|_{X_s}> \\
& = B^s(x)( {B^s}^TB^s + n \mathcal{P}^Q_s)^{-1}{B^s}^T (\mathcal{T}^sf)\\
& = B^s(x)( {B^s}^TB^s + n \mathcal{P}^Q_s)^{-1}{B^s}^T \mathbb{E}[Y] = \mathbb{E}[(A_sf)(x)]
\end{align*}

With $<Y, \hat{M}_{\Lambda_s}(x)> = B^s(x)( {B^s}^TB^s + n \mathcal{P}^Q_s)^{-1}{B^s}^T Y$
the proof is concluded
\end{proof}

\subsection{Consistency analysis}
In this section, we study the behavior of the of the problem formulation \ref{l2high}, when we relax the smoothness constraining enforced by the difference based penalty. The results in this section show that as we make the constraints inactive in our penalized formulation, the produced approximation tends to the unconstrained solution in the same RKHS, establishing the consistency of our constraining procedure.

\begin{definition}
Defining $\lambda^{\infty}_s \in \R$ as an upper bound to the set $\Lambda_s$ (other than the least upper bound) such that
\begin{equation}\label{lambainfty}
 \lim_{\lambda^{\infty}_s \to 0} (\lambda^i_s/\lambda^{\infty}_s) \to 0 \quad \forall i \in [1,d] \cap \mathbb{N}
\end{equation}
\end{definition}

Now, we will provide a corollary (to Theorem \ref{thcost}) explaining the behavior of $\hat{M}_{\Lambda_s}(x)$ as $\lambda^{\infty}_s$ tends to 0

\begin{corollary}
\label{corr}
The solution to the penalized objective (\ref{errfunc}) in the limit $\lambda^{\infty}_s \to 0$ is the orthogonal projection on the approximation space defined by $B^s$. Thus on solving

\begin{equation}\label{errfunc22}
\hat{M}_0(x) = 
\lim_{\lambda^{\infty}_s \to 0} \hat{M}_{\Lambda_s}(x) = \lim_{\lambda^{\infty}_s \to 0} \Big(\arg \min_{M_{\Lambda_s}(x) \in \R^n} \Big[||E^x_{\Lambda_s}||^2_{\mathcal{H}^{*}_s} + ||M_{\Lambda_s}(x)||^2_{T_s} \Big]\Big)
\end{equation}

we get $\hat{M}_0(x)  = B^s( {B^s}^TB^s)^{-1}R_s(x)|_{X_s}$ satisfying $<Y,\hat{M}_0(x)> = (A_sf)_0(x)$ .
\end{corollary}

\begin{proof}
The proof directly follows from Theorem \ref{thcost} and using (\ref{lambainfty}) as $\lambda^{\infty}_s \to 0$.
\end{proof}

In Corollary \ref{corr} we have mentioned the approximation $(A_sf)_0$, that is obtained by orthogonally projecting on $B^s$. Hence $(A_sf)_0(x) = B^s(x)({B^s}^TB^s)^{-1}{B^s}^TY$. Next, we provide a theorem relating $(A_sf)_0$ to $A_sf$. This result provides an understanding of the behavior of the produced approximation as constraints become active. However, before getting to the main results we start with a lemma. This lemma provides a tractable representation of inner product of the optimal approximation ($A_sf$) at scale s with any other function $\tilde{f}$ in the same space (note that $A_sf$,  $\tilde{f} \in \Gamma^s$).

\begin{lemma}\label{theorem2}
The weighted sum of inner products of projection components for $A_sf,\tilde{f} \in \Gamma^s$ along each penalized dimension, admits the Euclidean inner product representation

\begin{equation}\label{rep}
\sum_{i = 1}^d \lambda_s^i \Big<J^i_s(A_sf),J^i_s(\tilde{f}) \Big>_{\mathcal{H}_s}  =(1/n) \Big<Y -\mathcal{T}^s(A_sf), \mathcal{T}^s\tilde{f} \Big>
\end{equation}
\end{lemma}
\begin{proof}
We begin our proof by defining a semi-inner product $\Big<\cdot, \cdot \Big>_{\Lambda_s}$ on $\R^n \times \R^n$
\begin{equation}
    \Big<(U_1,U_2),(V_1,V_2) \Big>_{\Lambda_s}  = \frac{1}{n}\Big<U_1,V_1\Big> + \sum_{i = 1}^d\lambda^i_s \Big<J^i_s ({\mathcal{T}^s}^{\dagger}U_2),J^i_s ({\mathcal{T}^s}^{\dagger}V_2) \Big>_{\mathcal{H}_s}
    \label{normlam}
\end{equation}
Here $U_1,U_2,V_1,V_2 \in \R^n$. For it to be a valid norm we also assume $U_2,V_2 \in Dom({\mathcal{T}^s}^{\dagger})$ at scale s. 
Correspondingly we also obtain the semi-inner product induced semi-norm $||\cdot||_{\Lambda_s}$ on $\R^n \times \R^n$
\[
||(U,V)||^2_{\Lambda_s} = \frac{1}{n} ||U||^2_2  +\sum_{i = 1}^d\lambda^i_s||J^i_s({\mathcal{T}^s}^{\dagger}V)||^2_{\mathcal{H}_s}
\]
Now, it can be easily seen that the solution of the Regularized Network at scale s (problem (\ref{l2high})) can be expressed in $|| \cdot ||_{\Lambda_s}$ as follows.
\[
||(Y,0) - (\mathcal{T}^s(A_sf),\mathcal{T}^s(A_sf))||^2_{\Lambda_s} =\inf_{\tilde{f} \in \Gamma^s} ||(Y,0) - (\mathcal{T}^s\tilde{f},\mathcal{T}^s\tilde{f})||^2_{\Lambda_s}
\]
Therefore, since $(Y,0) - (\mathcal{T}^s(A_sf),\mathcal{T}^s(A_sf))$ would be orthogonal to all ($\mathcal{T}^s\tilde{f},\mathcal{T}^s\tilde{f})\in \R^n \times \R^n$ by the property of projections in finite dimensional spaces. Therefore,
\[
\Big<((Y,0) - (\mathcal{T}^s(A_sf),\mathcal{T}^s(A_sf)), (\mathcal{T}^s\tilde{f},\mathcal{T}^s\tilde{f})\Big>_{\Lambda_s} = 0\quad \forall \tilde{f} \in \Gamma^n
\]
\[
\implies \frac{1}{n} \Big<Y -\mathcal{T}^s(A_sf), \mathcal{T}^s\tilde{f} \Big> - \sum_{i=1}^d\lambda^i_s \Big<J^i_s(A_sf),J^i_s(\tilde{f}) \Big>_{\mathcal{H}_s} = 0
\quad \text{using } (\ref{normlam})\]
Thus, the result follows
\end{proof}

Coming back to the relation of $A_sf$ and $(A_sf)_0$, we now have the following first result

\begin{theorem}
For any $\Lambda_s = [\lambda_s^1, \lambda_s^2,....., \lambda_s^d] > 0 \in \R^d$, solution $A_sf$ to problem \ref{l2high} satisfies

\begin{itemize}
\item Pythagoras Theorem
\begin{equation}
\label{pyth}
||Y - \mathcal{T}^s(A_sf)||^2_2 = ||Y - \mathcal{T}^s(A_sf)_{0}||^2_2 + ||\mathcal{T}^s(A_sf)_{0} - \mathcal{T}^s(A_sf)|||^2_2
\end{equation}
\item Best approximation, if ${(A_sf)_0}|_X$ is observed instead of Y. Modifying (\ref{l2high})
\begin{equation}
A_sf =arg \min_{\tilde{f} \in \Gamma^s} \Bigg[ \frac{1}{n} ||\mathcal{T}^s(A_sf)_{0} - \mathcal{T}^s\tilde{f}||^2_2 + \sum_{i=1}^d\lambda^i_s ||J^i_s\tilde{f}||^2_{\mathcal{H}_s} \Bigg]
\end{equation}
\end{itemize}
\end{theorem}

\begin{proof}
(a): substituting $\Lambda_s = 0$ in Lemma \ref{theorem2}, we get $\Big<Y - \mathcal{T}^s(A_sf)_0, \mathcal{T}^s\tilde{f} \Big>  = 0$
Using this,
\begin{align*}
||Y -  \mathcal{T}^s\tilde{f}||^2_2 &= ||Y - \mathcal{T}^s(A_sf)_{0} +  \mathcal{T}^s(A_sf)_{0}    -  \mathcal{T}^s\tilde{f}||^2_2\\
&= ||Y   - \mathcal{T}^s(A_sf)_{0}||^2_2 + 2\Big<Y - \mathcal{T}^s(A_sf)_{0}, \mathcal{T}^s(A_sf)_{0}    -  \mathcal{T}^s\tilde{f} \Big> \\
&\ \ \ +||\mathcal{T}^s(A_sf)_{0}    - \mathcal{T}^s \tilde{f}||^2_2\\
&= ||Y - \mathcal{T}^s(A_sf)_{0}||^2_2 + ||\mathcal{T}^s(A_sf)_{0}    -  \mathcal{T}^s\tilde{f} ||^2_2
\end{align*}

Replacing $\mathcal{T}^s\tilde{f}$ by $\mathcal{T}^s(A_sf)$ completes the proof

(b): For proving the approximation property, we subtract $\Big<Y - \mathcal{T}^s(A_sf)_0, \mathcal{T}^s\tilde{f} \Big> =0$ from (\ref{rep}), we get 
\[
\frac{1}{n} \Big<Y - \mathcal{T}^s(A_sf), \mathcal{T}^s\tilde{f} \Big> - \frac{1}{n} \Big<Y - \mathcal{T}^s(A_sf)_{0} , \mathcal{T}^s\tilde{f}  \Big> =\sum_{i=1}^d \lambda^i_s \Big<J^i_s(A_sf),J^i_s\tilde{f} \Big>_{\mathcal{H}_s}
\]
Therefore, following Lemma \ref{theorem2}, $A_sf$ is again an optimal solution for the case when $ \mathcal{T}^s(A_sf)_{0}$ was observed instead of $Y$
\end{proof}

Again using the following result from Lemma \ref{theorem2},
\begin{equation}
\sum_{i = 1}^d \lambda_s^i \Big<J^i_s(A_sf),J^i_s(\tilde{f}) \Big>_{\mathcal{H}_s}  =(1/n) \Big<Y -\mathcal{T}^s(A_sf), \mathcal{T}^s\tilde{f} \Big>
\end{equation}

we now state our second main result that quantifies the rate of convergence of approximation $A_sf$ to $(A_sf)_0$ and $\mathcal{T}^s(A_sf)$ to $\mathcal{T}^s(A_sf)_0$, in $\mathcal{H}_s$ and n-dimensional Euclidean space respectively, as constraints are being rendered inactive.

\begin{theorem}
Approximations $A_sf$ and $\mathcal{T}^s(A_sf)$ converge to the unconstrained solutions $(A_sf)_0$ and $\mathcal{T}^s(A_sf)_0$ in $\mathcal{H}_s$ and $\R^n$ respectively as $\lambda^{\infty}_s \to 0$, according to the following convergence order (g is some finite positive constant).
\[
\lim_{\lambda^{\infty}_s \to 0} ||A_sf - (A_sf)_0||^2_{\mathcal{H}_s}  \leq \lim_{\lambda^{\infty}_s \to 0} ng^2\lambda^{\infty}_s \sum_{i=1}^d||J_s({\mathcal{T}^s}^{\dagger}Y) ||^2_{\mathcal{H}_s} \to 0 \quad :\mathcal{O}(\lambda^{\infty}_s) \text{ in }\mathcal{H}_s
\]
\[
\lim_{\lambda^{\infty}_s \to 0}\frac{2}{\lambda^{\infty}_s}||\mathcal{T}^s(A_sf) -\mathcal{T}^s(A_sf)_{0} ||_2^2 \to\ 0 \quad :{o}(\lambda^{\infty}_s)\text{ in }\mathbb{R}^n
\]
\end{theorem}

\begin{proof}
For any function $\tilde{f} \in \Gamma^s$, we define a norm as $
||\tilde{f}||_{\Gamma^s} = ||\mathcal{T}^s\tilde{f}||_2
$. Since $ \Gamma^s$ is finite dimensional, therefore norm $||\cdot||_{\Gamma^s}$ and $||\cdot||_{\mathcal{H}_s}$ would be equivalent on $\Gamma^s$ . Thus there would be a constant g ($>0$) such that
\begin{equation}\label{norm_relation}
||\tilde{f}||_{\mathcal{H}_s} \leq g||\tilde{f}||_{\Gamma^s}
\end{equation}
Using the result from Lemma \ref{theorem2} and substituting $\tilde{f} ={\mathcal{T}^s}^{\dagger}Y - A_sf$
\begin{multline}
\sum_{i=1}^d \lambda_s^i \Big<J^i_s(A_sf),J^i_s({\mathcal{T}^s}^{\dagger}Y) \Big>_{\mathcal{H}_s} - \sum_{i=1}^d\lambda_s^i\Big<J^i_s(A_sf),J^i_s(A_sf) \Big>_{\mathcal{H}_s} = \\
(1/n) \Big<Y -\mathcal{T}^s(A_sf),Y -\mathcal{T}^s(A_sf)\Big>
\end{multline}
On rearranging, we get
\begin{multline}\label{rearrange}
\frac{1}{n}||Y -\mathcal{T}^s(A_sf) ||_2^2 + 
\sum_{i = 1}^d\lambda^i_s||J^i_s(A_sf) \Big|\Big|^2_{\mathcal{H}_s} =\\
 \sum_{i=1}^d\lambda^i_s \Big<J^i_s(A_sf),J^i_s({\mathcal{T}^s}^{\dagger}Y) \Big>_{\mathcal{H}_s}
\leq  \sum_{i=1}^d\lambda^i_s ||J^i_s(A_sf)||_{\mathcal{H}_s} ||J^i_s({\mathcal{T}^s}^{\dagger}Y)||_{\mathcal{H}_s}
\end{multline}
Which directly leads to the inequality
\begin{equation}
\label{ineq_lam1}
\frac{1}{n}||Y -\mathcal{T}^s(A_sf) ||_2^2 \leq \sum_{i=1}^d\lambda^i_s ||J^i_s({\mathcal{T}^s}^{\dagger}Y) ||^2_{\mathcal{H}_s} \leq \lambda^{\infty}_s \sum_{i = 1}^d  ||J^i_s({\mathcal{T}^s}^{\dagger}Y) ||^2_{\mathcal{H}_s}
\end{equation}
Also, putting $\tilde{f} =A_sf - (A_sf)_{0}$ in (\ref{norm_relation}) and using (\ref{pyth}), we additionally get
\begin{equation}
\label{ineq_lam2}
||A_sf - (A_sf)_{0}||^2_{\mathcal{H}_s} \leq g^2||\mathcal{T}^s(A_sf) - \mathcal{T}^s(A_sf)_{0}||_2^2 \leq g^2||Y - \mathcal{T}^s(A_sf)||_2^2
\end{equation}
Using (\ref{ineq_lam1}) and (\ref{ineq_lam2}), the first statement of the theorem follows
\[
||A_sf - (A_sf)_{0}||^2_{\mathcal{H}_s} \leq ng^2 \lambda^{\infty}_s \sum_{i = 1}^d  ||J^i_s({\mathcal{T}^s}^{\dagger}Y) ||^2_{\mathcal{H}_s}
\]
For the second result, we begin with

\begin{align*}
\sum_{i=1}^d \lambda_s^i||J^i_s(A_sf - {\mathcal{T}^s}^{\dagger}Y)||^2_{\mathcal{H}_s} = \sum_{i=1}^d \lambda_s^i||J^i_s(A_sf) ||^2_{\mathcal{H}_s}  +\sum_{i=1}^d \lambda_s^i||J^i_s({\mathcal{T}^s}^{\dagger}Y) ||^2_{\mathcal{H}_s} - \\
2\sum_{i=1}^d \lambda_s^i\Big<J^i_s(A_sf),J^i_s({\mathcal{T}^s}^{\dagger}Y) \Big>_{\mathcal{H}_s} 
\end{align*}
On rearranging and using (\ref{rearrange}), we get
\begin{align*}
\sum_{i=1}^d \lambda_s^i||J^i_s(A_sf - {\mathcal{T}^s}^{\dagger}Y)||^2_{\mathcal{H}_s} = \sum_{i=1}^d \lambda_s^i||J^i_s(A_sf) ||^2_{\mathcal{H}_s}  +\sum_{i=1}^d \lambda_s^i||J^i_s({\mathcal{T}^s}^{\dagger}Y) ||^2_{\mathcal{H}_s} -\\
{2}\Bigg[\frac{1}{n}||Y -\mathcal{T}^s(A_sf) ||_2^2 + 
\sum_{i = 1}^d\lambda^i_s||J^i_s(A_sf) ||^2_{\mathcal{H}_s} \Bigg]
\end{align*}
On further solving and normalizing by $\lambda_s^{\infty}$ we get
\begin{align*}
\frac{2}{n \lambda_s^{\infty}}||Y -\mathcal{T}^s(A_sf) ||_2^2
= \sum_{i=1}^d \frac{\lambda_s^i}{\lambda_s^{\infty}}||J^i_s({\mathcal{T}^s}^{\dagger}Y) ||^2_{\mathcal{H}_s} - \\ \sum_{i=1}^d \frac{\lambda_s^i}{\lambda_s^{\infty}}||J^i_s(A_sf) ||^2_{\mathcal{H}_s} -
\sum_{i=1}^d \frac{\lambda_s^i}{\lambda_s^{\infty}}   ||J^i_s(A_sf - {\mathcal{T}^s}^{\dagger}Y)||^2_{\mathcal{H}_s} 
\end{align*}
Using (\ref{pyth})
\begin{align*}
\label{converge1}
\frac{2}{n \lambda_s^{\infty}}||\mathcal{T}^s(A_sf) -\mathcal{T}^s(A_sf)_0 ||_2^2
\leq \sum_{i=1}^d \frac{\lambda_s^i}{\lambda_s^{\infty}}||J^i_s({\mathcal{T}^s}^{\dagger}Y) ||^2_{\mathcal{H}_s} - \\ \sum_{i=1}^d \frac{\lambda_s^i}{\lambda_s^{\infty}}||J^i_s(A_sf) ||^2_{\mathcal{H}_s} -
\sum_{i=1}^d \frac{\lambda_s^i}{\lambda_s^{\infty}}   ||J^i_s(A_sf - {\mathcal{T}^s}^{\dagger}Y)||^2_{\mathcal{H}_s} 
\end{align*}
Now, looking at R.H.S of equation above and using (\ref{norm_relation})
\[
||J^i_s(A_sf - {\mathcal{T}^s}^{\dagger}Y)||_{\mathcal{H}_s} \leq ||J^i_s||_{\mathcal{H}_s}||{\mathcal{T}^s}^{\dagger}Y - A_sf||_{\mathcal{H}_s} \leq g||J^i_s||_{\mathcal{H}_s} ||Y - \mathcal{T}^s(A_sf)||_2
\]
Thus with $\lambda^{\infty}_s \to 0$ from part previous result $||J^i_s(A_sf - {\mathcal{T}^s}^{\dagger}Y)||_{\mathcal{H}_s} \to 0$. Also 
\[
||J^i_s({\mathcal{T}^s}^{\dagger}Y) ||_{\mathcal{H}_s}- ||J^i_s(A_sf) ||_{\mathcal{H}_s}  \leq ||J^i_s(A_sf)-J^i_s({\mathcal{T}^s}^{\dagger}Y) ||_{\mathcal{H}_s}
\]
Now, since $\lambda^i_s$ tends to 0 faster than $\lambda^{\infty}_s$, therefore we conclude
\[\frac{2}{n \lambda_s^{\infty}}||\mathcal{T}^s(A_sf) -\mathcal{T}^s(A_sf)_0 ||_2^2 \to\ 0\ as\ \lambda^{\infty}_s \to 0\]
Hence the proof follows
\end{proof}

\subsection{Bounding the approximation behavior}
In this section we analyze the behavior of the approximation produced at individual scales. We provide three results consisting of bounds on the (i) scale dependent approximation operator $A_s$ (ii) scale dependent approximation at a point $A_sf(x)$ (iii) scale dependent mean approximation error at a point $Error(x) = |f(x) - \mathbb{E}[A_sf](x)|$. The goal is to show that our formulation behaves in a stable manner as the model is trained to learn from data.

The starting result provides a bound for the approximation at any scale s with respect to the $L_{\infty}$ topology for a compact domain $\Omega \in \R^d$

\begin{theorem}
\label{stab2}
The approximation $A_sf$ has a $L_{\infty}$ upper bound
\begin{equation}
||A_sf||_{L_{\infty}} \leq P^s_\infty ||Y||_{\infty}
\end{equation}
with $P^s_{\infty}$ following the bounds
\begin{equation}
||U(\Lambda_s)M_0||_2 \leq P^s_{\infty} \leq ||U(\Lambda_s)M_0||_1
\end{equation}
where $U(\Lambda_s)$ is defined in (\ref{ulam}) and $M_0$ is from corollary \ref{corr} 
\end{theorem}
\begin{proof}
We begin with the definition of approximation $A_sf$ expressed as an inner product as in Theorem \ref{thcost}
\begin{align*}
||A_sf||_{L_{\infty}} &= \max_{x \in \Omega} |A_sf(x)| = \max_{x \in \Omega} \Big|\sum_{x_j \in X} y_j M^j_{\Lambda_s}(x)  \Big| \leq \max_{x \in \Omega}\sum_{x_j \in X} |y_j M^j_{\Lambda_s}(x) | \\
& \leq \max_{x \in \Omega}\sum_{x_j \in X} |y_j | \cdot |M^j_{\Lambda_s}(x) | \leq P^s_{\infty} ||Y||_{\infty}  \quad \text{ where } P^s_{\infty} =\max_{x \in \Omega} \sum_{j = 1}^n |M^j_{\Lambda_s}(x) |
\end{align*}

Now, for establishing bounds on $P^s_{\infty}$, we proceed as follows. Let $x^* \in \Omega$ be the data point at which the $\sum_{j =1}^n|M^j_{\Lambda_s}(x)|$ is maximized.
\begin{align*}
P^s_{\infty} = \sum_{j = 1}^n|M^j_{\Lambda_s}(x^*)| = \sum_{j = 1}^n|\delta^s_{x^*}  M^j_{\Lambda_s}| \leq \sum_{j = 1}^n ||\delta^s_{x^*}||_{\mathcal{H}^{*}_s} ||M^j_{\Lambda_s}||_{\mathcal{H}_s} =  \sum_{j = 1}^n||M^j_{\Lambda_s}||_{\mathcal{H}_s} 
\end{align*}

The last equality here comes from the assumed normalization : $||\delta^s_x||^2_{{\mathcal{H}}^{*}_s} =1$. Using the expression for $M^j_{\Lambda_s}$ from Theorem \ref{thcost}. 
\begin{multline}
<M^j_{\Lambda_s},M^j_{\Lambda_s}>_{\mathcal{H}_s} =
 e_j^T B^s({B^s}^TB^s + n\mathcal{P}_s^Q)^{-1}R_s|_{X_s} R_s^T|_{X_s}({B^s}^TB^s + n \mathcal{P}_s^Q)^{-1}{B^s}^Te_j
\end{multline}
\begin{multline}\label{mjmj}
= e_j^T B^s({B^s}^TB^s + n\mathcal{P}_s^Q)^{-1}{B^s}^TB^s ({B^s}^TB^s)^{-1}R_s|_{X_s}\\
 R_s^T|_{X_s}{B^s}^TB^s ({B^s}^TB^s)^{-1}({B^s}^TB^s + n \mathcal{P}_s^Q)^{-1}{B^s}^Te_j
\end{multline}
Now realizing 
\[
B^s({B^s}^TB^s + n\mathcal{P}_s^Q)^{-1}{B^s}^T = U(\Lambda_s)
\quad \text{and} \quad
B^s ({B^s}^TB^s)^{-1}R_s|_{X_s} = M_{0} 
\]
we get
\[
<M^j_{\Lambda_s},M^j_{\Lambda_s}>_{\mathcal{H}_s} =e_j^TU(\Lambda_s)M_0M_0^TU^T(\Lambda_s)e_j
\]
If, $U_j(\Lambda_s)$ represents the $j^{th}$ influence vector, then we get
\[
<M^j_{\Lambda_s},M^j_{\Lambda_s}>_{\mathcal{H}_s} =|<U_j(\Lambda_s),M_0>|^2
\implies
||M^j_{\Lambda_s}||_{\mathcal{H}_s} =|<U_j(\Lambda_s),M_0>|
\]
Therefore we get the upper bound on $P^s_{\infty}$ as
\begin{equation}
P^s_{\infty} \leq \sum_{j = 1}^n|<U_j(\Lambda_s),M_0>| =||U(\Lambda_s)M_0||_1
\end{equation}
For computing the lower bound, we again begin with the fact that,
\[
P^s_{\infty} =  \max_{x \in \Omega_s} ||M_{\Lambda_s}(x)||_1 \geq  \max_{x \in \Omega_s} ||M_{\Lambda_s}(x)||_2 = ||M_{\Lambda_s}||_2
\]
However from the computations for upper bound and Theorem \ref{thcost}, we infer
\[
<M_{\Lambda_s},M_{\Lambda_s}>_{\mathcal{H}_s}
=||U(\Lambda_s)M_0||^2_2
\]
Thus establishing the stated theorem
\end{proof}
Proceeding further we provide a result which bounds the approximation produced by the proposed approach at any data point $x \in \Omega_x$ and scale s

\begin{corollary}
The approximation at any $x \in \Omega_x$ is  bounded in the sense 

\begin{equation}
|A_sf(x)| \leq ||U(\Lambda_s)M_0||_1||Y||_{\infty}
\end{equation}
\end{corollary}

\begin{proof}
The proof follows similar steps to the previous theorem. Beginning with the inner product representation of the approximation
\begin{align*}
|A_sf(x)|  &= \sum_{j = 1}^n|y_jM^j_{\Lambda_s}(x)| \leq  \sum_{j = 1}^n|y_j|\cdot|M^j_{\Lambda_s}(x)| \leq ||Y||_{\infty} \sum_{j=1}^n |M^j_{\Lambda_s}(x)|
\end{align*}

Thus, by referring to the upper bound in Theorem \ref{stab2}, the result follows
\end{proof}

Now, as stated earlier, we provide bounds for the error in approximation at any new data point 

\begin{theorem}
The pointwise approximation error (in definition \ref{def43}) for any $x \in \Omega_x$
\begin{equation}
Error(x) = |f(x) - \mathbb{E}[A_sf](x)| = |\delta^s_x(f) - M^T_{\Lambda_s}(x)\delta^s_X (f)| = |E^x_{\Lambda_s}(f)|
\end{equation}
follows the upper bound
\[
Error(x) \leq (1-a) ||f||_{\mathcal{H}_s}
\]

Where $a = M_0^T(x)U(\Lambda_s)R_s(x)$
\end{theorem}

\begin{proof}
Starting with the the optimal value of $M_{\Lambda_s}(x)$ obtained in Theorem \ref{thcost}

\begin{equation}
    \label{stab4}
    M_{\Lambda_s}(x) = B^s({B^s}^TB^s + n \mathcal{P}_s^Q)^{-1}R(x)|_{X_s}
\end{equation}

Substituting it in the squared error functional norm

\begin{align*}
||E^x_{\Lambda_s}||^2_{\mathcal{H}^{*}_s}  &= < \delta^s_x - M^T_{\Lambda_s}(x)	 \delta^s_X,  \delta^s_x - M^T_{\Lambda_s}(x) \delta^s_X>_{\mathcal{H}^{*}_s}\\
&= ||\delta^s_x||^2_{{\mathcal{H}}^{*}_s} - 2M^T_{\Lambda_s}(x) \delta^s_X\delta^s_x + M^T_{\Lambda_s}(x) \delta^s_X{\delta^s_X}^T  M_{\Lambda_s}(x)\\
&=||\delta_x||^2_{{\mathcal{H}}^{*}_s} - 2M^T_{\Lambda_s}(x) R_s(x) + M^T_{\Lambda_s}(x) G_s M_{\Lambda_s}(x)
\end{align*}

Starting with the second term
\begin{align*}
M^T_{\Lambda_s}(x) R_s(x)  &= R(x)|^T_{X_s}({B^s}^TB^s + n \mathcal{P}_s^Q)^{-1}{B^s}^TR_s(x)\\
& = R(x)|^T_{X_s} ({B^s}^TB^s)^{-1}({B^s}^TB^s) ({B^s}^TB^s + n \mathcal{P}_s^Q)^{-1}{B^s}^TR_s(x)\\
&= M_0^T(x)U(\Lambda_s)R_s(x)
\end{align*}

Coming to the third term,$M^T_{\Lambda_s}(x) G_s M^s_{\Lambda_s}(x)$
\begin{align*}
&= R(x)|^T_{X_s}({B^s}^TB^s + n \mathcal{P}_s^Q)^{-1}{B^s}^TG_sB^s({B^s}^TB^s + n \mathcal{P}_s^Q)^{-1} R(x)|_{X_s}\\
&= M_0^T(x)U(\Lambda_s)R_s(x)R_s^T(x)U^T(\Lambda_s)M_0(x)
\end{align*}

Therefore 
\begin{equation}\label{opte}
||E^x_{\Lambda_s}||^2_{\mathcal{H}^{*}_s} = 1 - 2a + a^2 \quad \text{where }a = M_0^T(x)U(\Lambda_s)R_s(x)
\end{equation}

Now, coming back to the single point evaluation error representation as discussed earlier
\begin{equation}
Error(x) = |E^x_{\Lambda_s}(f)| \leq ||E^x_{\Lambda_s}||_{\mathcal{H}^{*}_s}||f||_{\mathcal{H}_s}
\end{equation}

Hence the result follows.
\end{proof}


\begin{figure}
\centering
\includegraphics[width=4.5in]{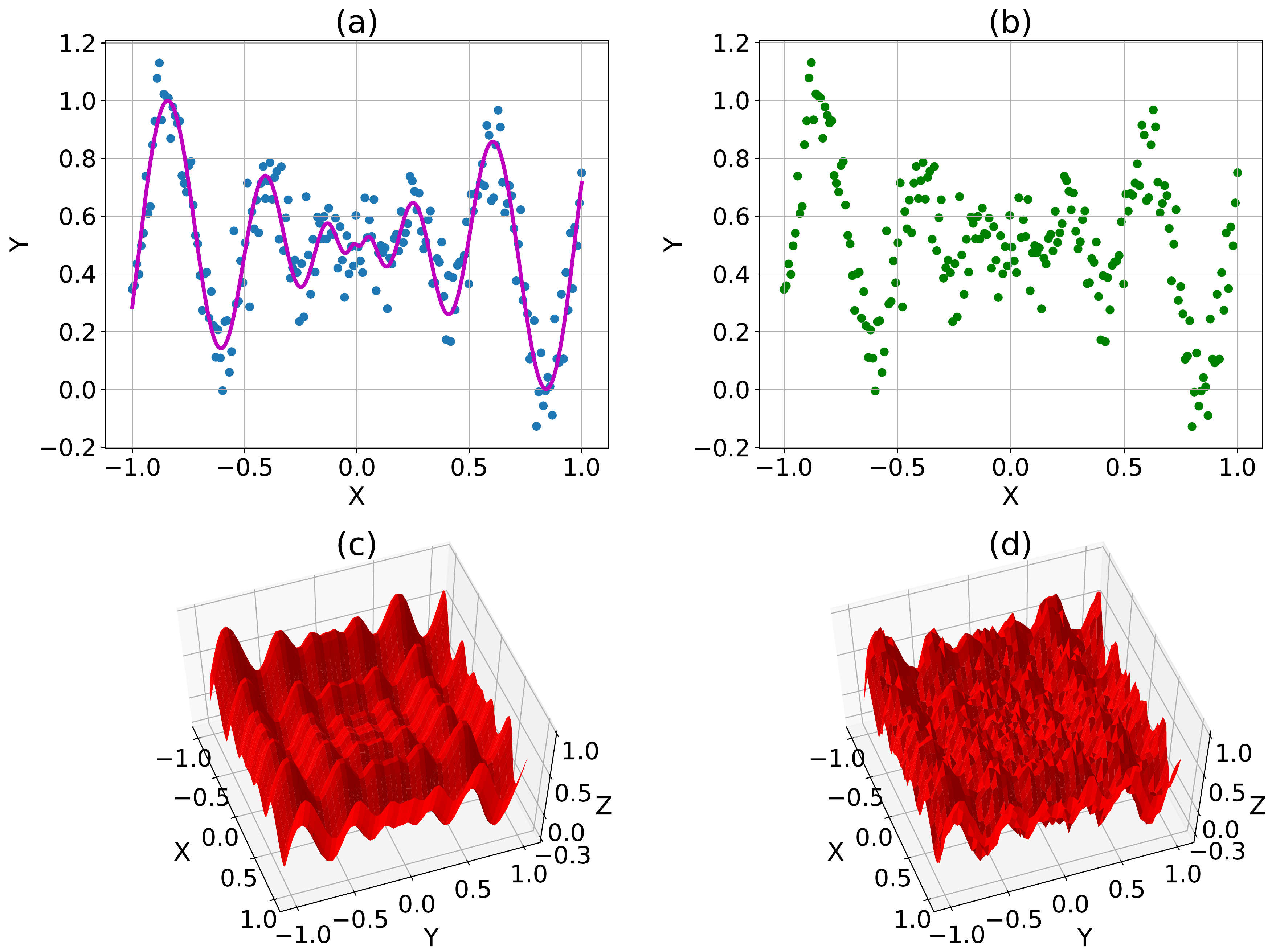}
\caption{{\color{red}(a) and (b)}: Univariate test function.
(a) shows the the 1-d Schwefel function along with the sampled noisy data. (b) here just shows this sampled data to give a visual intuition. {\color{red}(c) and (d)}: Multivariate test function. Here we show the bi-variate Bohachevsky function which we use for our analysis. Plot in (c) shows the true function whereas the plot (d) shows the noisy data sampled from it.}
\label{three_1}
\end{figure}

\begin{table}[]
    \centering
    \begin{tabular}{|c|c|c|c|c|c|c|c|}
         \hline
         \multicolumn{1}{|c|}{} &
         \multicolumn{3}{|c|}{} &
         \multicolumn{4}{|c|}{} \\
         \multicolumn{1}{|c|}{} &
         \multicolumn{3}{|c|}{\textbf{Univariate Function}} &
         \multicolumn{4}{|c|}{\textbf{Multivariate Function} }\\[3ex]
         \hline
          $Scale$ & $comp_s$ & $Cost_s$ & $q^{opt}$ & $comp_s$ & $Cost_s$ & $q_x^{opt}$ & $q_y^{opt}$\\ \hline
          \textit{0} & 0.94 & 4.99e-02 & 1 & 0.98 & 4.30e-02 & 1 & 1 \\ \hline
          \textit{1} & 0.93 & 3.26e-02 & 1 & 0.97 & 4.35e-02 & 1 & 1 \\ \hline
          \textit{2} & 0.92 & 2.22e-02 & 1 & 0.95 & 2.60e-02 & 1 & 1 \\ \hline
          \textit{3} & 0.90 & 1.83e-02 & 1 & 0.93 & 1.59e-02 & 1 & 1 \\ \hline
          \textit{4} & 0.87 & 1.13e-02 & 2 & 0.89 & 1.19e-02 & 1 & 1 \\ \hline
          \textit{5} & 0.82 & 1.09e-02 & 1 & 0.81 & 3.58e-03 & 2 & 1 \\ \hline
          \textit{6} & 0.77 & 1.10e-02 & 1 & 0.68 & 3.04e-03 & 1 & 1 \\ \hline
          \textit{7} & \textbf{0.68} & \textbf{1.06e-02} & \textbf{2} & \textbf{0.43} & \textbf{2.90e-03} & \textbf{2} & \textbf{1} \\ \hline
          \textit{8} & 0.57 & 1.07e-02 & 2 & 0.08 & 2.92e-03 & 1 & 1 \\ \hline
          \textit{9} & 0.40 & 1.10e-02 & 2 & 0.00 & 3.10e-03 & 1 & 2 \\ \hline
          \textit{10} & 0.18 & 1.13e-02 & 2 & - & - & - & - \\ \hline
          \textit{11} & 0.00 & 1.13e-02 & 2 & - & - & - & - \\ \hline
          
    \end{tabular}
    \caption{Performance of the proposed approach on Univariate (1d Schwefel) and Multivariate (Bohachevsky) test function. For Univariate test function, we have shown the compression ratio $comp_s$ (\ref{comp}), optimal cost at scale s (\ref{costgcv}) and optimal penalty order $q$ for all scales (0 to 11 as shown in Figure \ref{three_3}). For the Multivariate test function, the same analysis has been shown (with scales going from 0 to 9 as shown in Figure \ref{three_4}). The optimal penalties in X and Y direction is denoted by $q_x^{opt}$ and $q_y^{opt}$ respectively. Overall the scale with the minimum fitting cost ($Cost_s$) is highlighted (t = 7) for both cases.}
    \label{tab:table1}
\end{table}

\section{Results}

In this section, we present the results of the proposed hierarchical approach on univariate and multivariate synthetic datasets \cite{simulationlib} along with performance analysis on a time series dataset from remote sensing literature \cite{luthcke2013antarctica}. This makes sense as simulated datasets can test the modeling capability with respect to the truth and application on real datasets can test the behavior of the proposed method on the challenges which come with the real observations. 

Firstly we begin with the application on two test functions (shown in Figure \ref{three_1}). The univariate function here shows noisy data sampled from the 1-d Schwefel function \cite{simulationlib}(in (a) and (b)) . The non-convexity of this function coupled with sharp curvature changes is expected to pose a good challenge for any noisy data modeling procedure. The multivariate function here ((c) and (d)) pose similar challenges but in higher dimensions.

\begin{figure}
\centering
\includegraphics[width=5in]{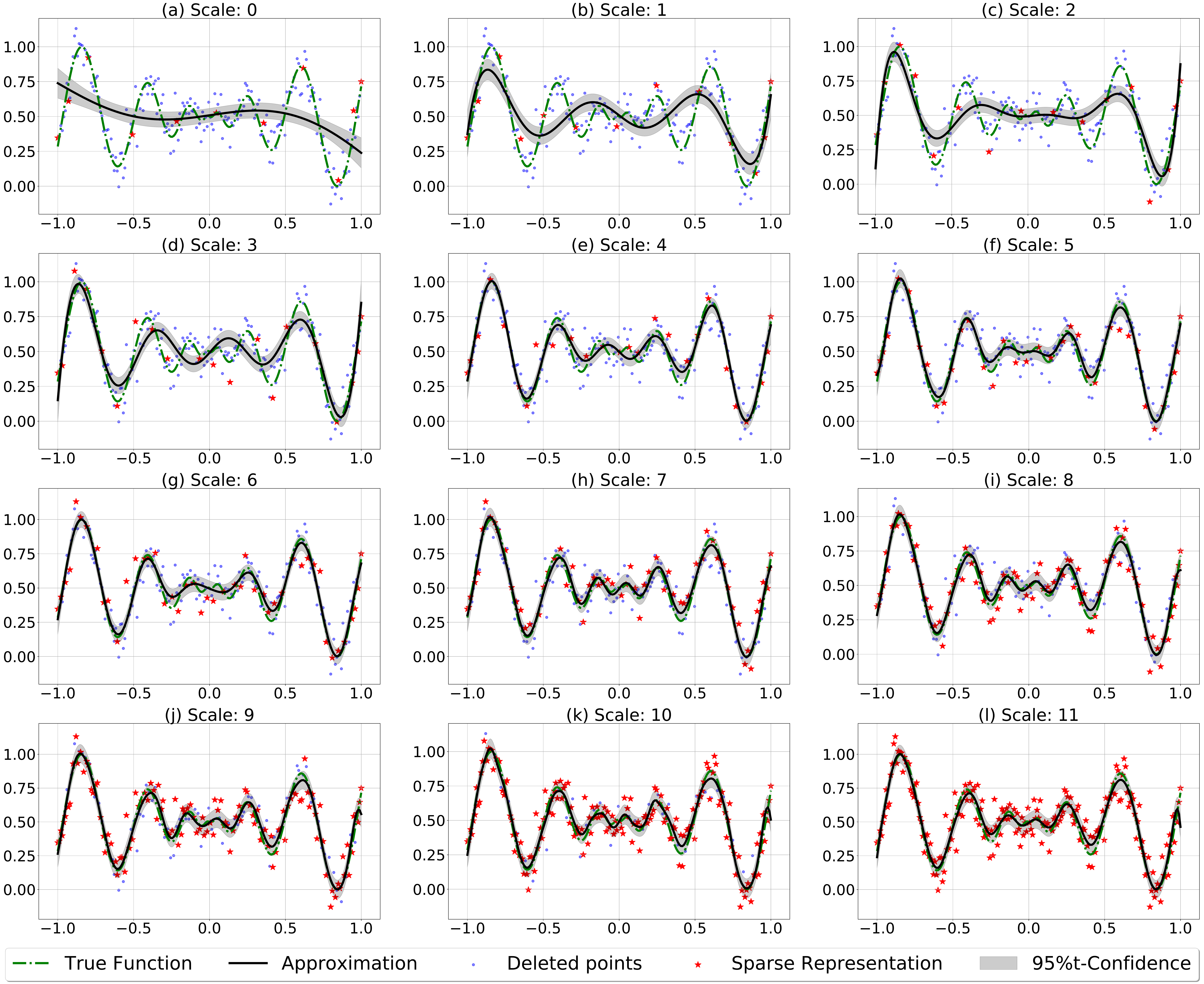}
\caption{Scale-wise performance and solution of the proposed approach on univariate test function. With smaller sparse representations, the approximation is oversmoothed at initial scales with noticeable improvement as the scale increases. Scale 7 here produces the best approximation. The legends are shown at the bottom of the figure.}
\label{three_3}
\end{figure}

\subsection{Understanding the behavior with scales}

Considering the univariate test function, Figure \ref{three_3} provides an intuitive understanding of the behavior of the approach across the scales. Here, starting with scale 0, we show that at each scale increment, more and more points are chosen in the sparse representation leading to the corresponding improvement in the produced approximation. Here we also compute compression ratio at scale s ($comp_s$) defined as
\begin{equation}
comp_s = 1 - \frac{l_s}{n} \quad ;(l_s\text{ is the cardinality of }X_s)
\label{comp}
\end{equation}

\begin{figure}
\centering
\includegraphics[width=4in]{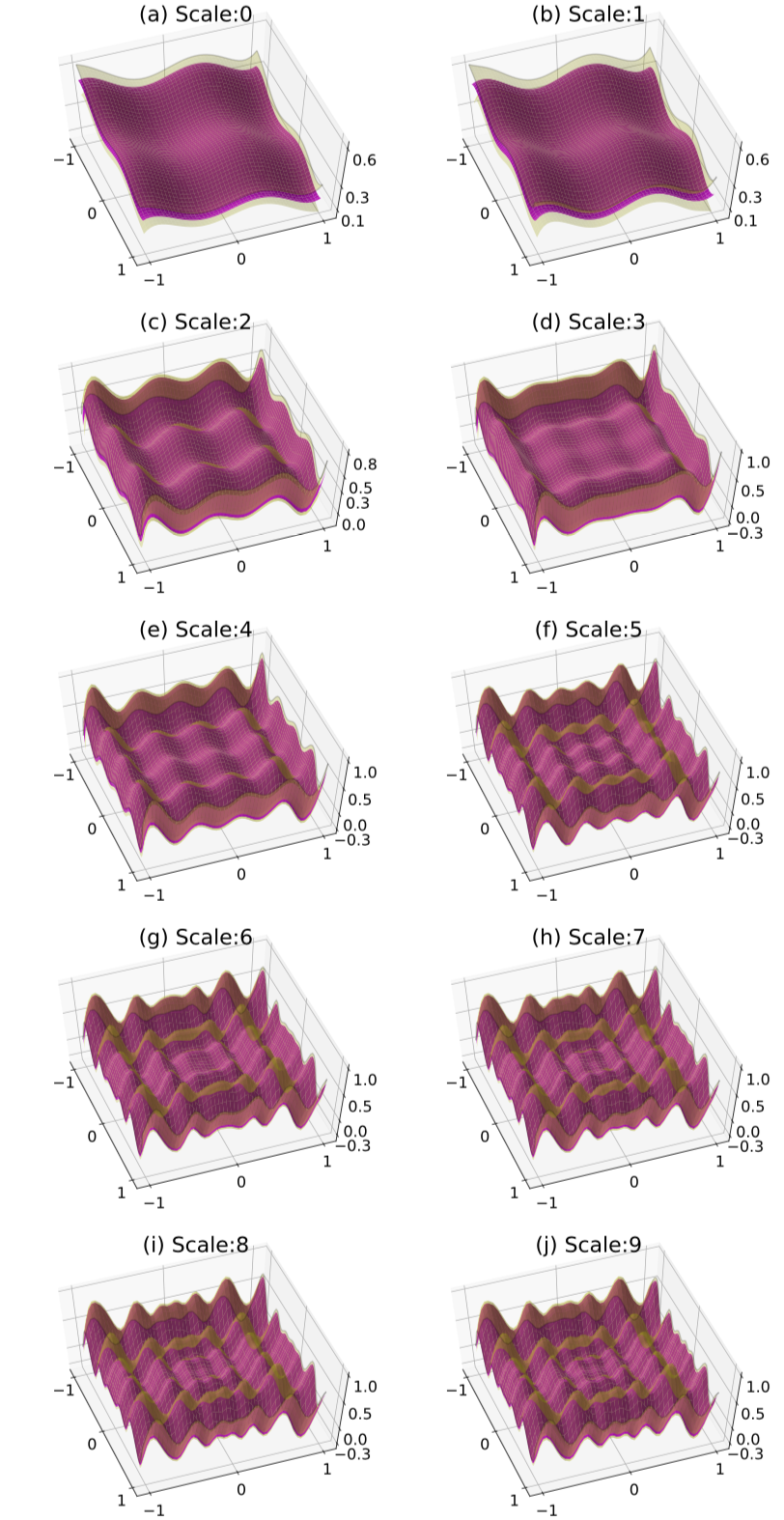}
\caption{Scale-wise performance and solution of proposed approach on bi-variate test function. The light surface above and below the mean approximation (magenta colored) shows the 95\% t-confidence intervals.}
\label{three_4}
\end{figure}

Therefore a value of $comp_s$ closer to 1 shows that very few observations were selected in the sparse representation and hence represents good compression being achieved. Starting with scale 0 (Figure \ref{three_3}), the cost of fitting quantified as the optimal GCV value was observed to achieve a minima at scale 7 (details are shown in Table \ref{tab:table1}), establishing it as the convergence scale (t). This is also evident from the quality of the approximation produced at scale 7 (in Figure \ref{three_3})). Moreover it should be noted that the cost of fitting at convergence scale was even less than cost of fitting with the full datasets (Table \ref{tab:table1}). This is intuitive since here we are trying to find a trade-off between model complexity and generalization capability.

\begin{figure}
\centering
\includegraphics[width=5in]{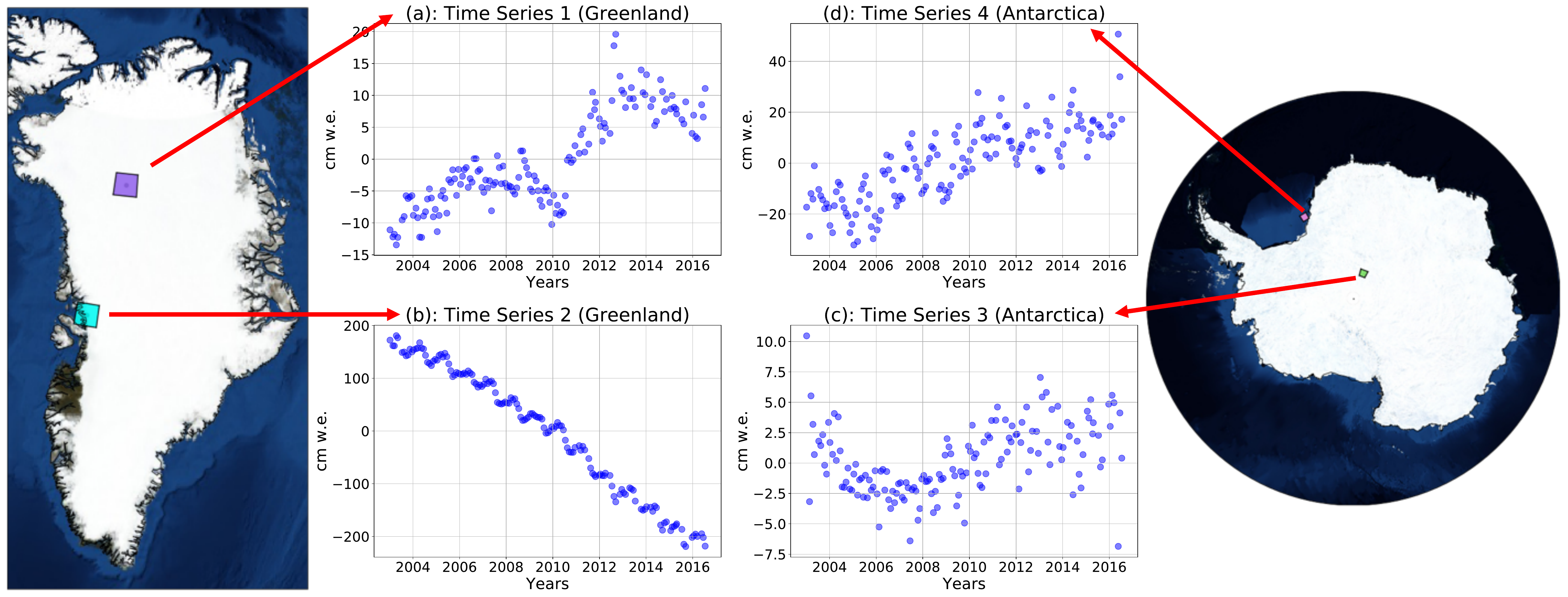}
\caption{Location of the test time series on the Greenland (time series 1 in (a) and 2 in (b)) and Antarctic (time series 3 in (c) and 4 in (d)) ice sheets. We have chosen the time series from both accumulation (near the center with more frequent snowing) and ablation zones (near the edge with higher degree of fluctuations and activity) of the ice sheets for testing the proposed approach}
\label{three_6}
\end{figure}

Moving forward with the bi-variate test case, here, we show a similar analysis in Figure \ref{three_4}. Here, the transparent surfaces  sandwiching the mean approximation show the $\pm$ $95 \% t$-confidence intervals.

For  better understanding of the performance and behavior of  the algorithm on the two test functions, we have presented the scalewise performance details in Table \ref{tab:table1}. Here we show the compression ratio (\ref{comp}) achieved with different scales along with the optimal penalty order chosen at each scale of analysis for both the test functions ($q^{opt}$ for univariate and $q_x^{opt}$, $q_y^{opt}$ for multivariate case respectively). It should be noted here that for the multivariate case (Table \ref{tab:table1}), we have shown the optimal penalty order in both X and Y direction (which does not necessarily have to be the same).

\subsection{Application on real data}

\begin{figure}
\centering
\includegraphics[width=5in]{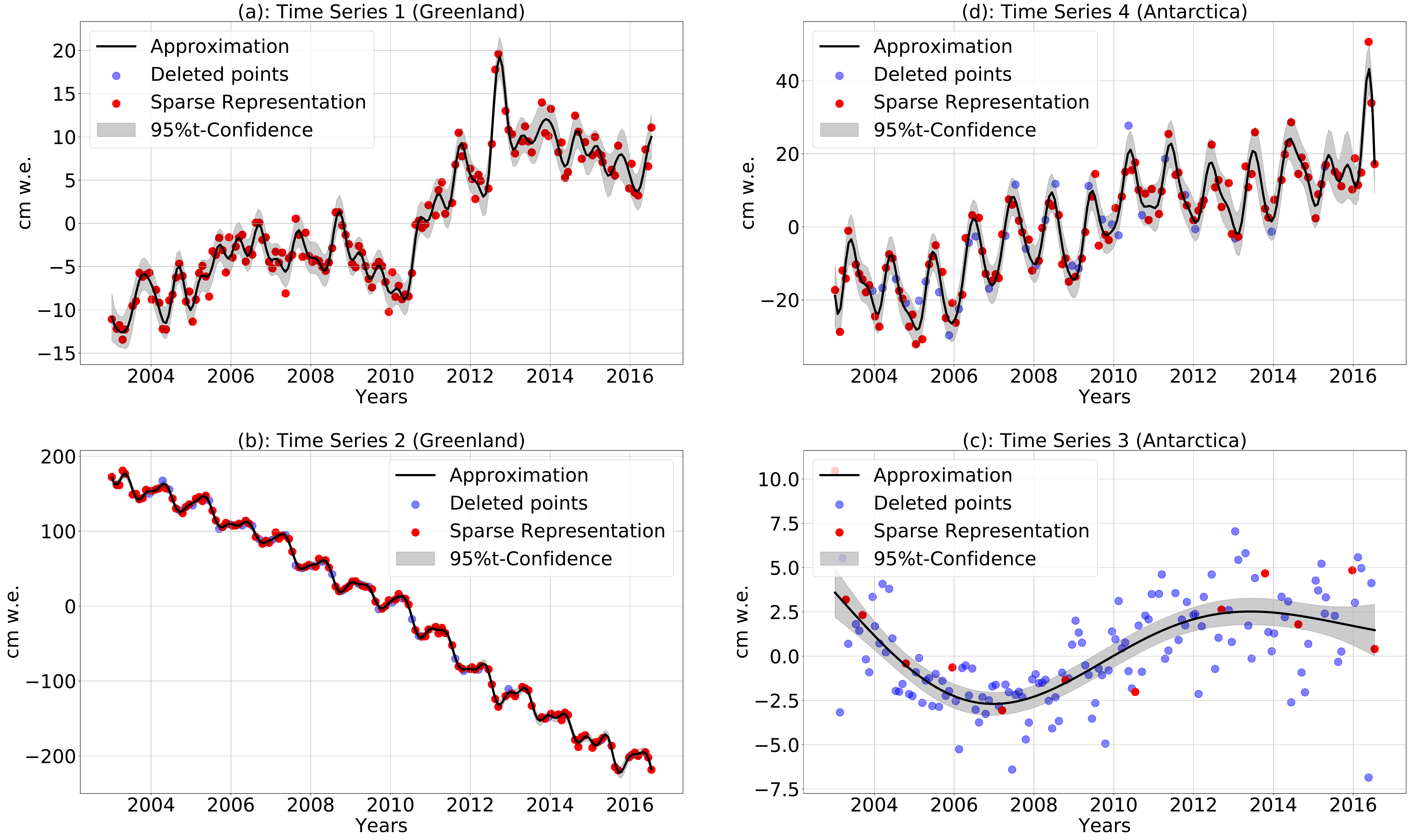}
\caption{Performance on the 4 time series from Greenland and Antarctic Icesheets (shown in Figure \ref{three_6}) with 95\%t-confidence intervals.}
\label{three_7}
\end{figure}

Here we consider the application of our approach on time series of cm. equivalents of water height . These time series were derived in \cite{luthcke2013antarctica} with the objective of studying changes in mass of ice around the globe (with regions divided broadly as ice sheets, ice shelves, land and water). For our purpose we consider 4 different time series here as shown in Figure \ref{three_6}. Here time series 1 and 2 are from Greenland showing the accumulation and ablation (melting) behavior respectively. Time series 3 and 4 show this behavior for Antarctic ice sheet. Figure \ref{three_7} then shows the approximation produced by our approach on these time series. 

For time series 1, the approach is able to capture a rich structure from previous noisy looking data. Here one other important thing to  note is that all the points were selected in the sparse representation to produce the best possible approximation. This further shows the nature of the approach to prefer good approximation over a simpler model. For time series 2, we have a clear periodicity in the structure of the data which is suitably captured by our approach. Moving further, time series 3 again shows one very important property of our approach. Here since the data is very noisy, hence the sparse representation chosen is very small as compared to the full dataset. This is because of the lack of structure in the data and hence a simpler model leads to a better generalization performance. In the last time series (time series 4), the algorithm again captures the periodicity in the data while choosing a subset of the dataset as the optimal sparse representation for generating approximations. The compression ratios and the optimal penalty order for the test time series are shown in Table \ref{tab:my_label2}.

\section{Conclusion}

In this paper, we presented a hierarchical regularization network based approach to generate sparse representations for noisy datasets with Generalized Cross Validation (GCV) for model selection and fitting. We  provided a detailed theoretical framework for the approach particularly studying the approximation behavior coupled with consistency and convergence.

These sparse representations were also shown to act as a model for the datasets to produce good approximations at previously un-observed data points. For testing the procedure, test datasets were picked from both simulations and observed real data repositories. On all of these datasets the approach was found to perform well providing an inference for the approximation with confidence intervals from the generated sparse representations.  

\begin{table}[]
    \centering
    \begin{tabular}{|c||*{4}{c|}}\hline
        \backslashbox{Feature}{Time Series}
        &\makebox[1.2em]{TS1}&\makebox[1.2em]{TS2}&\makebox[1.2em]{TS3}
        &\makebox[1.2em]{TS4}\\\hline\hline
        Convergence Scale ($t$) from 0 to 10 &10&9&1&9\\\hline
        Compression ratio at $s = t$ &0.00&0.22&0.91&0.22\\\hline
        Optimal Penalty at $s = t$ ($q^{opt}$) &2&2&1&1\\\hline
    \end{tabular}
    \caption{Performance details on the time series data (Figure \ref{three_6}). Here we have used the acronym TS for Time Series.}
    \label{tab:my_label2}
\end{table}

The next steps of this approach to sparse modeling with data reduction will be to extend the approach to very large datasets through efficient distributed implementations and intelligent data structures.  The quantification of model uncertainty could also be further improved by Bayesian sampling approaches that can effectively propagate the uncertainty of scale selection and inference of other parameters to the final model outcome. These  are expected to be a part of our future works.

\bibliographystyle{amsplain}
\bibliography{Shekhar}

\end{document}